\documentclass[11pt]{article}

\usepackage[final]{acl}

\usepackage{times}
\usepackage{latexsym}

\usepackage[T1]{fontenc}

\usepackage[utf8]{inputenc}

\usepackage{microtype}

\usepackage{inconsolata}

\usepackage{graphicx}
\usepackage{hyperref}       
\usepackage{url}            
\usepackage{booktabs}       
\usepackage{amsfonts}       
\usepackage{nicefrac}       
\usepackage{microtype}      
\usepackage{lipsum}		    
\usepackage{graphicx}
\usepackage{natbib}
\usepackage{doi}
\usepackage{amsmath}
\usepackage{wrapfig}
\usepackage{titletoc}
\usepackage{amsmath}
\usepackage{amssymb}
\usepackage{mathtools}
\usepackage{amsthm}
\usepackage{bm}
\usepackage{multirow}
\usepackage{tcolorbox}
\tcbuselibrary{breakable}
\usepackage[toc]{appendix}
\usepackage{minitoc}
\usepackage[capitalize,noabbrev]{cleveref}
\usepackage{authblk}
\usepackage{enumitem}
\usepackage{float}
\usepackage{makecell}
\usepackage[table]{xcolor}
\usepackage{subcaption}

\usepackage{colortbl}
\usepackage{tcolorbox}
\tcbuselibrary{breakable}

\theoremstyle{plain}
\newtheorem{theorem}{Theorem}[section]
\newtheorem{proposition}[theorem]{Proposition}
\newtheorem{lemma}[theorem]{Lemma}

\theoremstyle{definition}
\newtheorem{definition}[theorem]{Definition}

\theoremstyle{remark}
\newtheorem{remark}[theorem]{Remark}

\newcommand{\ourmethod}{R\textsuperscript{2}-MAD}
\newcommand{\fullmethod}{\textbf{R}emember and \textbf{R}eweight for \textbf{M}ulti-\textbf{A}gent \textbf{D}ebate}
\newcommand{\madmm}{MAD-M\textsuperscript{2}}

\title{Remember and Reweight: Enhancing Multi-Agent Debate with \\ Experience Memory and Confidence Estimation}

\author{
{\bf Xuanfa Jin}$^{1,2}$ \quad {\bf Zhijian Ma}$^{1,2}$ \quad {\bf Yongcheng Zeng}$^{1,2}$ \quad {\bf Xinyu Cui}$^{1,2}$ \\ {\bf Haifeng Zhang}$^{1,2*}$ \quad {\bf Jun Wang}$^3$\thanks{Corresponding authors.} \\

$^1$Institute of Automation, Chinese Academy of Sciences \\
$^2$University of Chinese Academy of Sciences \quad $^3$University College London \\

\texttt{\{jinxuanfa2022,mazhijian2024,zengyongcheng2022,cuixinyu2021\}@ia.ac.cn} \\
\texttt{haifeng.zhang@ia.ac.cn,jun.wang@cs.ucl.ac.uk}
}

\begin{document}

\maketitle

\begin{abstract}
Multi-agent debate (MAD) improves the reasoning capabilities of large language models by having multiple agents iteratively refine their responses through discussion. However, MAD suffers from a critical vulnerability known as shared misconception: when a majority of agents initially converge on an incorrect answer, the debate process tends to amplify rather than correct the error. Existing methods primarily address peer skew but leave the agents' inherently biased concept priors unaddressed. To mitigate this systematic weakness, we propose \textbf{\ourmethod{}} (\fullmethod{}), a framework that equips agents with an experience memory accumulated from past debates. \ourmethod{} intervenes on both failure modes through two complementary mechanisms: A debate-state-aware retrieval policy dynamically calibrates the concept prior by retrieving relevant historical evidence based on the current consensus level. Then these retrieved experiences provide a basis for estimating per-agent reliability, yielding confidence weights to modulate peer influence. Experiments on various benchmarks show that \ourmethod{} achieves consistent improvements over existing single-agent and MAD baselines.
\end{abstract}

\section{Introduction}
Multi-agent debate (MAD) has emerged as a promising paradigm for improving the reasoning capabilities of large language models (LLMs) \citep{du2024improving, liang2024encouraging}. By having multiple LLM-based agents iteratively discuss and refine their responses, MAD has demonstrated consistent gains over single-agent baselines across a range of tasks, including reasoning \citep{zhu2026demystifying, ling2025memad}, evaluation \citep{chan2024chateval}, and problem solving \citep{li2025swe}. The underlying intuition is appealing: diverse agents can challenge each other's errors, and iterative refinement drives the group toward better answers. However, this optimistic picture obscures a critical vulnerability. When a majority of agents happen to initially converge on the same incorrect answer, which we refer to as \emph{shared misconception} \citep{estornell2024multi}, the debate process does not merely fail to correct the error but tends to amplify it. Agents that originally held the correct answer are progressively persuaded to abandon their stances in favor of the majority's consensus. This failure mode is not an occasional anomaly but a systematic weakness of the MAD framework \citep{wynn2025talk}. To mitigate it, we must first understand why debate fails under this regime and which factors contribute to the amplification of the initial error.

\begin{figure}
    \centering
    \includegraphics[width=\linewidth]{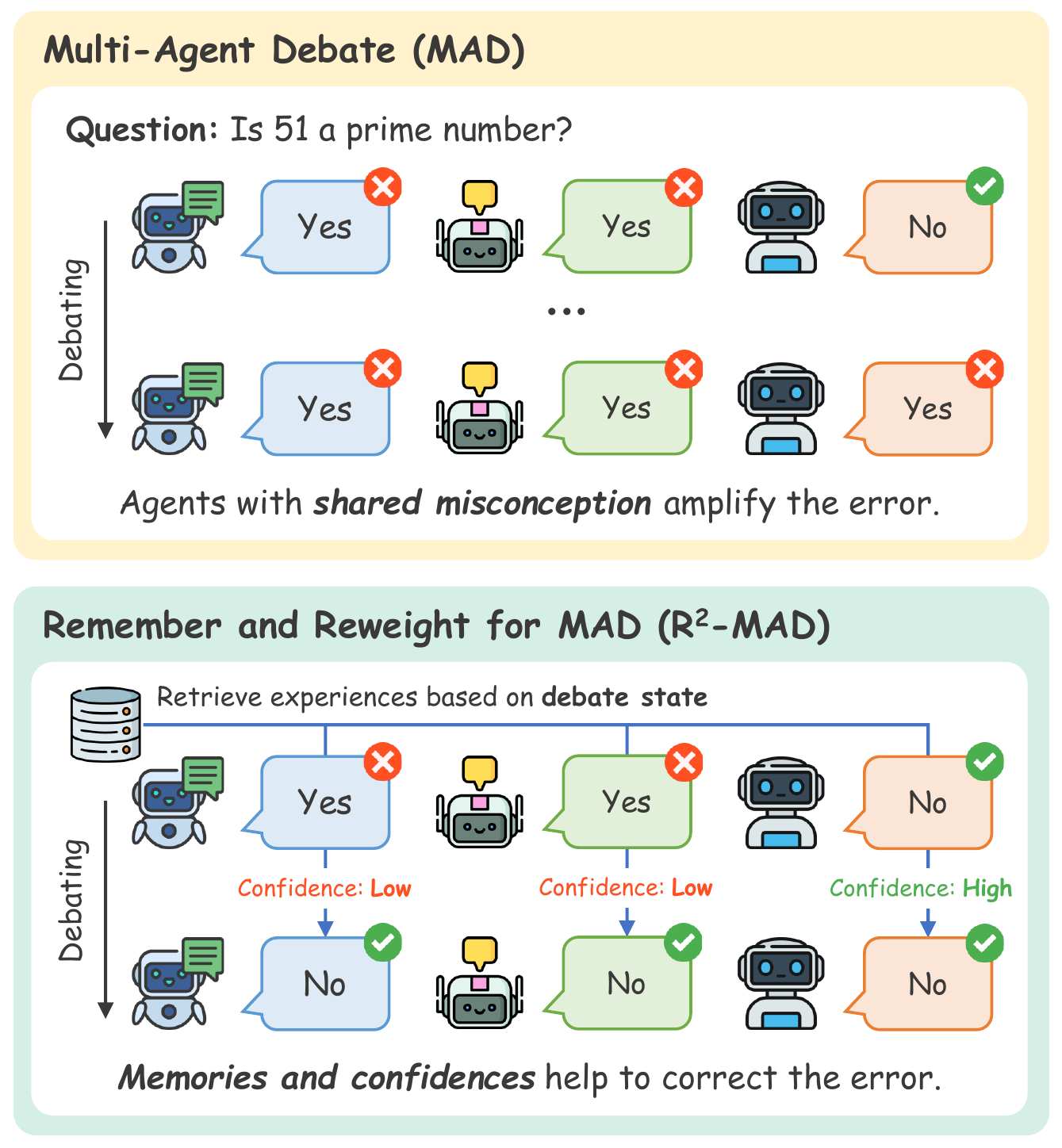}
    \vspace{-0.6cm}
    \caption{An illustrative example of the shared misconception problem. In standard MAD (top), the two agents holding the incorrect answer persuade the correct agent to switch, resulting in a wrong consensus. In \ourmethod{} (bottom), retrieved experiences and confidence weighting help agents resist the erroneous majority and converge to the correct answer.}
    \label{fig:example}
    \vspace{-0.6cm}
\end{figure}

Recent theoretical analysis provides useful insight into this question. \citet{estornell2024multi} show that under the latent concept framework, each agent's generation during debate can be decomposed into two factors: a \emph{concept prior}, which encodes the agent's intrinsic belief about the correct answer, and a \emph{peer skew}, which captures the cumulative influence of other agents' responses. Under shared misconception, these two factors compound: the prior is already biased toward the erroneous concept, and the peer skew amplifies this bias at a rate proportional to the number of agreeing agents. This analysis suggests a natural design principle: effective mitigation should intervene on both factors.  But existing approaches \citep{li2024sparse, liu2024groupdebate, tian2026multi} focus on manipulating other agents' responses to reduce the peer skew, leaving the biased prior unaddressed. This motivates us to explore mechanisms that can simultaneously correct the prior and modulate the skew, and we find that experiences from past debates offer a natural substrate for both.

Based on this insight, we propose \textbf{\ourmethod{}} (Remember and Reweight for Multi-Agent Debate), a framework that equips each agent with an experience memory accumulated from past debates. The memory serves two complementary purposes. First, agents retrieve relevant experiences to calibrate their prior beliefs with historical evidence. The retrieval is guided by a debate-state-aware policy that adapts to the current level of consensus: when consensus is high, retrieval favors diverse and contrastive experiences that challenge the majority; when agents remain divided, retrieval prioritizes experiences associated with positive outcomes. Second, by examining how each agent's stance has performed in historically similar situations, the retrieved experiences provide a basis for estimating per-agent reliability, which is then used as confidence weights to modulate peer influence during debate. The two mechanisms target the prior and the skew, respectively, and can be naturally combined within a unified framework.

We evaluate \ourmethod{} on various benchmarks. Experimental results show that \ourmethod{} outperforms single-agent methods and achieves consistent improvements over existing MAD baselines. Ablation studies confirm that both the memory-based prior correction and the confidence weighting contribute to the overall gains. Our main contributions are summarized as follows: (1) We propose \ourmethod{}, a framework that mitigates shared misconceptions in multi-agent debate by leveraging experience memory to intervene on both the concept prior and the peer skew. (2) We provide supplementary theoretical analysis that extends the existing latent concept decomposition framework to accommodate memory and confidence weighting, offering formal justification for the proposed design. (3) We conduct extensive experiments on multiple benchmarks, demonstrating the effectiveness of \ourmethod{} and validating the contribution of each component.\footnote{Code is available in \url{https://github.com/KylJin/R2-MAD}.}

\section{Related Work}
\paragraph{Multi-Agent Debate.}
The multi-agent debate (MAD) framework leverages collaborative interactions among multiple LLM-based agents to improve reasoning \citep{du2024improving}. Following this paradigm, prior work has explored divergent debate \citep{liang2024encouraging}, LLM-as-judge evaluation \citep{chan2024chateval}, efficient communication structures \citep{li2024sparse}, and other applications \citep{chen2024reconcile,jin2024learning}. To further improve debate quality, several methods have been proposed to refine how agents interact: diversity pruning removes near-duplicate responses to encourage broader exploration \citep{estornell2024multi}, selective masking filters unreliable messages from previous rounds to prevent error propagation \citep{tian2026multi}, and FreeMaD replaces rigid turn-taking with flexible, asynchronous updates \citep{cui2025freemad}. Other studies have diagnosed remaining failure modes, including uncalibrated confidence expression and insufficient agent diversity \citep{lin2025enhancing, zhu2026demystifying}, as well as costly message passing that motivates equilibrium-based formulations \citep{yi2025debate}. Complementary to these efforts, \ourmethod{} introduces a cross-debate perspective by using experiences accumulated from past debates to calibrate both agent beliefs and peer influence.

\paragraph{Memory-Based LLM Agents.}
Memory has become an important component of LLM agents. Previous work has used memory to store observations and reflections for more believable long-term behavior \citep{park2023generative}, feedback for iterative self-improvement \citep{shinn2023reflexion}, reusable skills acquired from interaction \citep{wang2023voyager}, and long-term user or interaction histories \citep{zhong2024memorybank, packer2023memgpt}. Beyond storing past interactions, recent memory-augmented methods have improved how memories are retrieved, compressed, and organized, ranging from memory-inspired retrieval \citep{qian2024memorag} and question-reflection memory \citep{wang2024qrmem} to reversible compression \citep{wang2025r3mem} and agent-level dynamic memory organization \citep{xu2025amem}. When it comes to MAD, \madmm{} \citep{tian2026multi} tries to mask unreliable memories from previous debate rounds to prevent error propagation, while MeMAD \citep{ling2025memad} stores structured debate transcripts and retrieves relevant past experiences to guide future reasoning. Compared to them, \ourmethod{} further leverages retrieved experiences to estimate per-agent reliability for confidence weighting, and dynamically adjusts its retrieval strategy based on the evolving debate state.

\begin{figure*}[t]
  \centering
  \includegraphics[width=\linewidth]{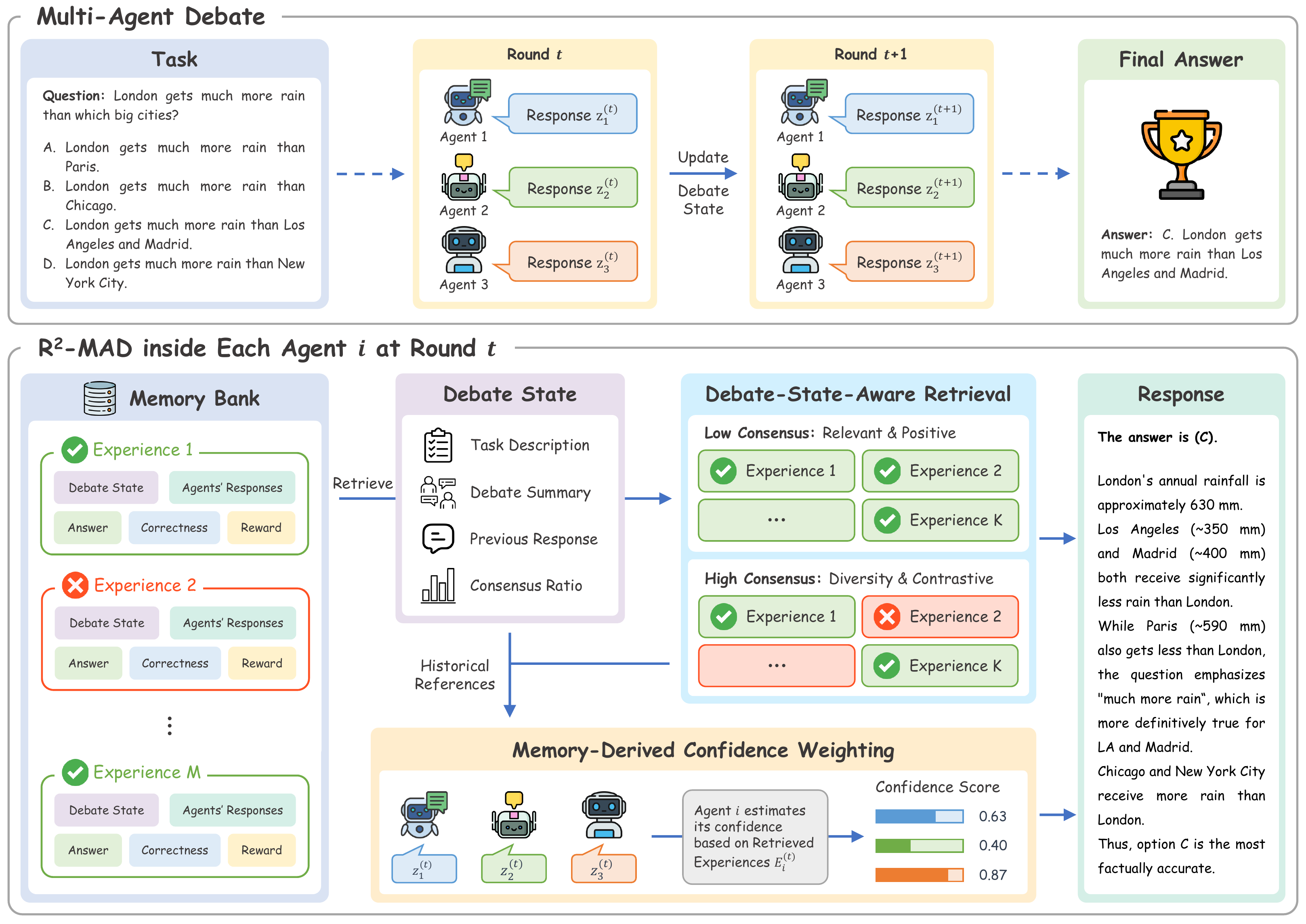}
  \vspace{-0.7cm}
  \caption{Overview of the \ourmethod{} framework. At each debate round, agent \(i\) retrieves relevant experiences from its memory bank via a debate-state-aware policy, uses them to calibrate its prior concepts and estimate confidence weights for each peer, and then generates an updated response.}
  \label{fig:framework}
  \vspace{-0.5cm}
\end{figure*}

\section{Preliminary}
\paragraph{Multi-Agent Debate Framework.}
Given a task \(x\) with a target answer \(y\), the Multi-Agent Debate (MAD) framework \citep{du2024improving} aims to leverage a collective of \(n\) LLM-based agents to iteratively resolve the task over \(T\) rounds. Let \(\phi_i\) denote the parameters of agent \(i\) (\textit{e.g.}, model weights, training data, or prior contexts). At the initial round \(t = 0\), each agent \(i\) independently generates a response \(z_i^{(0)}\) based on the input \(x\). For subsequent rounds \(0 < t \leq T\), agent \(i\) updates its stance by observing the joint responses of all agents from the previous round \(Z^{(t-1)} = (z_1^{(t-1)}, \dots, z_n^{(t-1)})\). Formally, the generation probability of agent \(i\) at round \(t\) is defined as:
\begin{equation}
    \mathbb{P} \left( z_i^{(t)} \Big| x, Z^{(t-1)}, \phi_i \right),
\end{equation}
where both the contextual input \((x, Z^{(t-1)})\) and the internal parameter \(\phi_i\) govern the generation process. This iterative debate continues until a maximum horizon \(T\) is reached or a consensus is established. Finally, to extract answers from agents' responses, we define \(a(\cdot)\) as an answer extraction function.

\paragraph{Latent Concept Decomposition.}
To model the underlying reasoning process in multi-agent debate, the generation dynamics can be analyzed in terms of a latent concept space \(\Theta\) \citep{xie2021explanation}. Under this paradigm, each task-answer pair \((x, y)\) is assumed to be generated from a true latent concept \(\theta^\star \in \Theta\). By introducing a conditional independence assumption, where an agent's response \(z_i^{(t+1)}\) depends solely on the latent concept \(\theta\) and its parameters \(\phi_i\) once \(\theta\) is given, \citet{estornell2024multi} derives a \emph{skew decomposition} for the generation probability:
\begin{align} \label{eq:estornell}
    & \mathbb{P} \left( z_i^{(t+1)} \Big| x, Z^{(t)}, \phi_i \right) \propto \sum_{\theta \in \Theta} \Big[ \mathbb{P} \left( z_i^{(t+1)} \Big| \theta, \phi_i \right) \notag \\
    & \mathbb{P} \left( x | \theta, \phi_i \right) \mathbb{P} \left( \theta | \phi_i \right) \prod_{j=1}^{n} \mathbb{P} \left( z_j^{(t)} \Big| \theta, \phi_i \right) \Big].
\end{align}
Here, the formulation explicitly factorizes the debate process into two distinct components: the agent's intrinsic generation capability independent of its peers, and the coordination bias or skew introduced by peer interactions.

\section{\ourmethod{} Framework}
\label{sec:method}

\subsection{Overview}
The latent concept decomposition reveals two distinct factors governing an agent's generation during debate: the \emph{concept prior} \(\mathbb{P}(\theta | \phi_i)\), which encodes the agent's intrinsic belief, and the \emph{peer skew} \(\prod_j \mathbb{P}(z_j^{(t)} | \theta, \phi_i)\), which captures the cumulative influence of other agents. When a shared misconception arises, both factors work against correctness. The prior is already biased toward an erroneous concept \(\theta^\prime\), and the skew amplifies this bias as more agents echo the same mistake. 

To address both failure modes, we propose \textbf{\ourmethod{}} (Remember and Reweight for Multi-Agent Debate), which augments each agent with an experience memory populated from past debates. Under a natural conditional independence assumption (detailed in Appendix~\ref{appendix:proof}), the generation probability under our framework is
\begin{align} \label{eq:joint_gen}
    & \mathbb{P}_E \left( z_i^{(t+1)} \Big| x, Z^{(t)}, E_i^{(t)}, \phi_i \right) \notag \\
    \propto & \sum_{\theta \in \Theta} \bigg[ \mathbb{P} \left( z_i^{(t+1)} \Big| \theta, \phi_i \right) \mathbb{P} \left( x | \theta, \phi_i \right) \notag \\
    & \mathbb{P} \left( \theta \Big| E_i^{(t)}, \phi_i \right) \prod_{j=1}^{n} \mathbb{P} \left( z_j^{(t)} \Big| \theta, \phi_i \right)^{w_{j,i}^{(t)}} \bigg].
\end{align}

Compared with the vanilla decomposition (Eq.~\eqref{eq:estornell}), our proposed framework introduces two modifications. First, the concept prior \(\mathbb{P}(\theta | \phi_i)\) is replaced by a memory-corrected prior \(\mathbb{P}(\theta | E_i^{(t)}, \phi_i)\), where \(E_i^{(t)} \subseteq M_i\) denotes experiences retrieved from agent \(i\)'s memory bank \(M_i\) via a debate-state-aware retrieval policy (\cref{sec:retrieval}). Second, each peer's contribution to the skew term is modulated by a confidence weight \(w_{j,i}^{(t)}\), which is also estimated from agent \(i\)'s retrieved experiences (\cref{sec:conf_weight}). The two modifications act on distinct components of the decomposition, and the overall framework is illustrated in \cref{fig:framework}.

\subsection{Debate-State-Aware Memory Retrieval} \label{sec:retrieval}
Retrieving memories based solely on task similarity, as commonly adopted in memory-based LLM agents \citep{ling2025memad, tian2026multi}, cannot capture the dynamics of an ongoing debate or adapt to agents' evolving needs across rounds. For instance, when a majority of agents have agreed on the same answer, the consensus may reflect a shared misconception instead of correctness, and the most valuable memories are those in which a similar majority turned out to be wrong. Conversely, when agents remain divided, the priority shifts to retrieving memories that are both relevant and historically associated with the positive outcomes, helping agents identify the potential right direction. Hence, we first define agent \(i\)'s \emph{debate state} to characterize its current status, and then design a debate-state-aware retrieval policy that dynamically adjusts its strategy accordingly.

\paragraph{Debate State.}
Before generating a response at round \(t\) (\(t > 0\)), each agent \(i\) observes the current debate state, defined as the tuple
\begin{equation} \label{eq:debate_state}
    s_i^{(t)} := \langle x, z_i^{(t-1)}, h^{(t-1)}, \mathrm{Cons}(Z^{(t-1)}) \rangle,
\end{equation}
where \(x\) is the current task, \(z_i^{(t-1)}\) is agent \(i\)'s previous response, \(h^{(t-1)}\) is a natural language summary of the previous round's debate, and \(\mathrm{Cons}(Z^{(t-1)})\) is the consensus ratio, defined as the fraction of agents agreeing on the most popular answer:
\begin{equation} \label{eq:consensus_ratio}
    \mathrm{Cons}(Z) = \max_y \frac{1}{n} \sum_{i=1}^n \mathbf{1} \left[ a(z_i) = y \right].
\end{equation}

\paragraph{Experience Memory.}
After a debate concludes, each agent \(i\) will extract key information from each round \(t\) to construct memory cases \(e_{i}^{(t)}\) and store them in its memory bank \(M_i\). Specifically, the structure of round \(t\)'s memory case \(e_{i}^{(t)}\) is formally defined as:
\begin{equation} \label{eq:memory_case}
        e_{i}^{(t)} := \langle s_{i}^{(t)}, Z^{(t)}, y, \zeta_{i}^{(t)}, r_{i} \rangle.
\end{equation}
Here, \(Z^{(t)}\) is all agents' responses at round \(t\), \(y\) is the answer to task \(x\), \(\zeta_{i}^{(t)}\) indicates the correctness of agent \(i\)'s current response, and \(r_{i}\) is the outcome reward for agent \(i\) (\(1\) for correct, \(0\) for wrong). We note that \(r_i\) only requires a signal indicating whether a concluded debate reached a good answer, which any method that learns from debate outcomes necessarily needs. This signal is not tied to gold annotations: it can equally be provided by a verifier model or by execution feedback on verifiable tasks. In our experiments, it simply reuses the labels already in the benchmarks, requiring no annotation beyond the datasets themselves. Over time, the memory bank \(M_i\) accumulates a growing collection of debate experiences that can be drawn upon in future debates.

\paragraph{Retrieval Policy.}
Given the current debate state \(s_i^{(t)}\) and memory bank \(M_i\), the state-aware retrieval policy proceeds in two stages. Assuming to retrieve \(K\) cases from the memory, we first sample a candidate set \(\mathcal{C}_i^{(t)}\) by retrieving top-\(3K\) cases based on the cosine similarity between \(s_i^{(t)}\) and the debate state in memory cases. We then select \(K\) cases from \(\mathcal{C}_i^{(t)}\) via Maximal Marginal Relevance (MMR), starting from \(E_i^{(t)} = \emptyset\) and greedily selecting at each step:
\begin{align} \label{eq:mmr}
    e^\star =\ & {\arg\max}_{e \in \mathcal{C}_i^{(t)} \backslash E_i^{(t)}} \left[ \lambda^t \cdot \mathrm{sim}(e, s_i^{(t)}) \cdot r_i(e) \right. \notag \\
    & \left. - (1 - \lambda^t) \max_{e^\prime \in E_i^{(t)}} \mathrm{sim}(e, e^\prime) \right],
\end{align}
until the size of the retrieved memory set \(| E_i^{(t)} | = K\). Here, \(r_i(e)\) is the outcome reward for memory case \(e\), and the trade-off coefficient \(\lambda^t\) is a function of the consensus ratio:
\begin{equation}
    \lambda^t = 1 - \gamma\,\mathrm{Cons}(Z^{(t-1)}), \ \gamma \in [0, 1].
\end{equation}
The first term in Eq. \eqref{eq:mmr} combines relevance with outcome quality: the \(\mathrm{sim}(e, s_i^{(t)}) \cdot r_i(e)\) assigns high scores to experiences that are both semantically related and positive. The second term penalizes redundancy by discouraging experiences similar to those already selected. The consensus-dependent \(\lambda^t\) dynamically balances the two objectives: 1) when consensus is low, \(\lambda^t\) is large and retrieval favors relevant, positive experiences to help agents reach consensus; 2) when consensus is high, \(\lambda^t\) decreases, allowing diverse and contrastive experiences to enter the retrieved set.

From a theoretical perspective, the retrieved experiences act as a likelihood-ratio correction to the agent's concept prior, updating \(\mathbb{P}(\theta | \phi_i)\) to \(\mathbb{P}(\theta | E_i^{(t)}, \phi_i)\). We formalize this as follows (refer to Appendix~\ref{proof:prior-correction} for detailed proof):
\begin{tcolorbox}[left=1.5mm, right=1.5mm, top=1.5mm, bottom=1.5mm, breakable]
\begin{proposition}[Memory as prior correction]
\label{prop:prior-correction}
    The memory-corrected prior satisfies 
    \begin{equation}
        \mathbb{P}(\theta | E_i^{(t)}, \phi_i) = P(\theta | \phi_i) \cdot \frac{\mathbb{P}(E_i^{(t)} | \theta, \phi_i)}{\mathbb{P}(E_i^{(t)} | \phi_i)}.
    \end{equation}
    That is, the retrieved experiences \(E_i^{(t)}\) act as a likelihood-ratio correction to the agent's bare prior. In particular, if \(E_i^{(t)}\) provide more evidence for the true concept \(\theta^\star\) than for an erroneous concept \(\theta^\prime\), i.e., \(\mathbb{P}(E_i^{(t)} | \theta^\star, \phi_i) > \mathbb{P}(E_i^{(t)} | \theta^\prime, \phi_i) \), then
    \begin{equation}
        \frac{\mathbb{P}(\theta^\star | E_i^{(t)}, \phi_i)}{\mathbb{P}(\theta^\prime | E_i^{(t)}, \phi_i)} > \frac{\mathbb{P}(\theta^\star | \phi_i)}{\mathbb{P}(\theta^\prime | \phi_i)},
    \end{equation}
    directly counteracting the bias toward \(\theta^\prime\).
\end{proposition}
\end{tcolorbox}

\subsection{Memory-Derived Confidence Weighting} \label{sec:conf_weight}
In a standard multi-agent debate, all agents' responses are presented equally to their peers, regardless of their reliability. However, an agent whose current stance has historically led to negative outcomes in similar circumstances should carry less influence, while an agent whose stance has been consistently positive deserves greater weight. Rather than requiring an external oracle to judge agent quality, we leverage the same retrieved experiences from Eq.~\eqref{eq:mmr} to estimate each agent's reliability and weight their influence accordingly.

\paragraph{Confidence Estimation.}
Given agent \(i\)'s retrieved experiences \(E_i^{(t)} = \{e_{i,1}, \dots, e_{i,K}\}\), since these experiences share similar debate states with agent \(i\) (according to the retrieval policy), we estimate the reliability of agent \(j\)'s current stance by comparing it against its historical stances recorded in each experience. Specifically, for each experience \(e_k\), we compute a soft matching weight based on the cosine similarity between agent \(j\)'s current response and its historical responses in the retrieved experiences, then aggregate these weights with its correctness to obtain a confidence score:
\begin{equation}
    c_{j,i}^{(t)} = \frac{\sum_{k=1}^K \mathrm{sim}(z_j^{(t)}, z_j(e_{i,k})) \cdot \zeta_j(e_{i,k})}{\sum_{k=1}^K \mathrm{sim}(z_j^{(t)}, z_j(e_{i,k}))},
\end{equation}
where \(z_j(e_{i,k})\) and \(\zeta_j(e_{i,k})\) denote agent \(j\)'s response and correctness in agent \(i\)'s experience \(e_{i,k}\). Intuitively, \(c_{j,i}^{(t)}\) reflects agent \(i\)'s estimate of agent \(j\)'s reliability: by examining how agent \(j\) performed in historically similar debate states, agent \(i\) assigns higher confidence when agent \(j\)'s past responses in similar situations were more often correct, and lower confidence otherwise.

\paragraph{Confidence Weighting.}
The confidence score determines how much influence agent \(j\) should exert in agent \(i\)'s subsequent generation. Formally, \(c_{j,i}^{(t)}\) is mapped to a confidence weight \(w_{j,i}^{(t)}\) via a sigmoid function such that \(c_{j,i} > 0.5\) yields \(w_{j,i} > 1\) (amplified influence), and \(c_{j,i} < 0.5\) yields \(w_{j,i} < 1\) (suppressed influence). Since the token probabilities of a black-box LLM cannot be edited directly, we realize \(w_{j,i}^{(t)}\) at the prompt level by annotating each agent's response based on two thresholds \(w_h\) and \(w_l\) (set to \(0.55\) and \(0.45\) by default). When \(c_{j,i}^{(t)} > w_h\), agent \(j\)'s response is marked by agent \(i\) as high confidence, signaling that this stance has been historically reliable in similar debate states. When \(c_{j,i}^{(t)} < w_l\), the response is marked as low confidence, indicating limited historical support. Otherwise, the historical evidence is considered inconclusive, and agent \(j\)'s response is presented without explicit confidence annotation.

\citet{estornell2024multi} show that when \(m\) agents produce responses aligned with a common erroneous concept \(\theta^\prime\), the debate posterior \(\mathbb{P}(z_i^{(t)} | x, Z^{(t)}, \phi_i)\) becomes dominated by \(\theta^\prime\) at a rate of \(O(m)\), a phenomenon termed \emph{majority dominance}. Our confidence weighting mechanism can mitigate this effect (proof in Appendix~\ref{proof:anti-majority}):
\begin{tcolorbox}[left=1.5mm, right=1.5mm, top=1.5mm, bottom=1.5mm, breakable]
\begin{proposition}[Anti-majority dominance]
\label{prop:anti-majority}
    Suppose \(m\) agents share a common misconception \(\theta^\prime\) and are assigned confidence weight \(\alpha \in (0, 1)\), while the remaining agents receive weight \(1\). The rate at which the debate posterior converges to \(\theta^\prime\) is slowed from \(O(m)\) to \(O(\alpha m)\).
\end{proposition}
\end{tcolorbox}

We further show that the memory lift and confidence lift are additive in log-odds space, providing a joint guarantee that, under the stated conditions, combining both mechanisms is at least as effective as using either alone. The formal statement and proof are provided as Theorem~\ref{thm:joint} in Appendix~\ref{proof:joint}.

\begin{table*}[t]
\centering
\begin{tabular}{l | c c c c | c}
    \toprule
    \textbf{Methods} & \textbf{MATH500} & \textbf{Economics} & \textbf{Engineering} & \textbf{TruthfulQA} & \textbf{Average} \\
    \midrule
    \multicolumn{6}{c}{\textit{Qwen2.5-7B-Instruct}} \\
    \midrule
    CoT & \textbf{0.530} & 0.649 & 0.407 & 0.675 & 0.565 \\
    Self-Consistency & \underline{0.522} & \underline{0.668} & \underline{0.477} & \underline{0.704} & \underline{0.593} \\
    MAD & 0.515 & 0.645 & 0.444 & 0.681 & 0.571 \\
    \madmm{} & 0.403 & 0.635 & 0.374 & 0.590 & 0.501 \\
    \rowcolor{cyan!15} \ourmethod{} (ours) & \underline{0.522} & \textbf{0.701} & \textbf{0.481} & \textbf{0.723} & \textbf{0.607} \\
    \midrule
    \multicolumn{6}{c}{\textit{Qwen3-8B}} \\
    \midrule
    CoT & 0.627 & 0.787 & 0.465 & 0.807 & 0.672 \\
    Self-Consistency & 0.671 & \underline{0.791} & 0.535 & \underline{0.825} & 0.706 \\
    MAD & \underline{0.821} & 0.787 & 0.634 & \textbf{0.843} & \underline{0.771} \\
    \madmm{} & 0.769 & 0.787 & \textbf{0.667} & 0.705 & 0.732 \\
    \rowcolor{cyan!15} \ourmethod{} (ours) & \textbf{0.843} & \textbf{0.796} & \underline{0.638} & \textbf{0.843} & \textbf{0.780} \\
    \midrule
    \multicolumn{6}{c}{\textit{Gemma-3-4B-IT}} \\
    \midrule
    CoT & 0.500 & 0.446 & 0.198 & 0.627 & 0.443 \\
    Self-Consistency & \underline{0.530} & 0.488 & 0.255 & 0.632 & 0.476 \\
    MAD & 0.500 & 0.502 & 0.243 & \underline{0.692} & \underline{0.484} \\
    \madmm{} & \textbf{0.537} & \underline{0.507} & \underline{0.259} & 0.448 & 0.448 \\
    \rowcolor{cyan!15} \ourmethod{} (ours) & \underline{0.530} & \textbf{0.521} & \textbf{0.263} & \textbf{0.705} & \textbf{0.505} \\
    \bottomrule
\end{tabular}
\vspace{-0.1cm}
\caption{Overall accuracy of different methods across four benchmarks using three large language models. \textbf{Bold} and \underline{underlined} values indicate the highest and second-highest accuracies in each model's column.}
\label{tab:main_results}
\vspace{-0.4cm}
\end{table*}

\section{Experiments}
\label{sec:exp}

\subsection{Experiment Setups}
\label{sec:exp_setup}
\paragraph{Benchmarks.}
To comprehensively evaluate the performance of \ourmethod{}, we conduct experiments on four benchmarks that cover both reasoning and knowledge-intensive tasks: MATH500 \citep{lightman2024let} for mathematical reasoning, Economics and Engineering subsets from MMLU-Pro \citep{wang2024mmlu} for domain-specific knowledge reasoning, and TruthfulQA \citep{lin2022truthfulqa} for factual judgment. This combination allows us to assess whether \ourmethod{} generalizes across tasks with different characteristics and patterns.

\paragraph{Baselines.}
We compare our methods against the following baselines: (1) \textbf{Chain of Thought (CoT)} \citep{wei2022chain}, where a single agent performs chain-of-thought reasoning; (2) \textbf{Self-Consistency} \citep{wang2022self}, which samples multiple reasoning paths from a single agent and selects the answer by majority voting; (3) \textbf{MAD} \citep{du2024improving}, the standard multi-agent debate framework; and (4) \textbf{\madmm{}} \citep{tian2026multi}, which augments MAD with memory masking to filter unreliable information from previous rounds. To further isolate whether our gains simply follow from access to related past experiences, we additionally compare against (5) \textbf{ICL-CoT}, a single-agent baseline that retrieves from the same memory bank as \ourmethod{} and uses the retrieved experiences as in-context examples; results are reported in Appendix~\ref{appendix:icl_cot}.

\paragraph{Implementation Details.}
We mainly conduct experiments with three open-source LLMs: Qwen2.5-7B-Instruct \citep{yang2024qwen2}, Qwen3-8B \citep{yang2025qwen3}, and Gemma-3-4B-IT \citep{gemma2025gemma3}. To further test whether the benefit persists on larger models, we also evaluate it on Llama-3.3-70B-Instruct \citep{grattafiori2024llama} and GPT-4o-mini \citep{openai2024gpt4}. For all debate-based methods, we use 3 agents over 3 rounds of debate. To prevent test data from leaking into agents' memory, we split each dataset into a training set for memory construction and a test set for evaluation. For MATH500, we use Level-5 (hardest) problems as the test set and the remaining as the training set. For other datasets, we randomly sampled a portion of the data as the test set. To ensure that no method benefits from prompt-level differences, the system, task, and debate prompts listed in Appendix~\ref{appendix:prompts} are shared identically across all methods. Further details are available in Appendix~\ref{appendix:impl_details}.

\subsection{Results and Analysis}
\label{sec:exp_results}
We organize our experimental analysis around four research questions: \textbf{RQ1} examines the overall performance of \ourmethod{} against existing baselines; \textbf{RQ2} validates the individual contribution of each proposed component through ablation; \textbf{RQ3} isolates the contribution of the debate-state-aware retrieval policy; and \textbf{RQ4} investigates whether \ourmethod{} is effective under the shared misconception regime, which is the core motivation of this work.

\paragraph{RQ1: Does \ourmethod{} improve over existing single-agent and MAD baselines?}
To answer this question, we conduct experiments across four benchmarks and three models. The results are presented in \cref{tab:main_results}. \ourmethod{} achieves the highest average accuracy on all three models, consistently outperforming both single-agent methods and debate baselines, with the largest improvements on the knowledge-intensive benchmarks such as Economics and TruthfulQA and somewhat limited on MATH500, where the answer hinges on a long derivation rather than on domain knowledge. Which baseline comes closest varies by model: MAD is the runner-up on Qwen3-8B and Gemma-3-4B-IT, whereas on Qwen2.5-7B-Instruct Self-Consistency becomes the strongest baseline once its sampling budget is matched to that of the debate methods. We also note that \madmm{} performs inconsistently, falling below standard MAD on several benchmarks (e.g., MATH500 and TruthfulQA on Qwen2.5-7B-Instruct), suggesting that its memory masking strategy can sometimes discard useful information. In contrast, \ourmethod{}'s selective retrieval and confidence weighting provide more robust improvements. Beyond these three open-source models, \ourmethod{} also attains the best average accuracy against CoT and MAD on Llama-3.3-70B-Instruct and GPT-4o-mini, indicating that the benefit is not confined to the small-model regime. The full results are in Appendix~\ref{appendix:large_models}.

\begin{table*}[t]
\centering
\begin{tabular}{l | c c c c | c}
    \toprule
    \textbf{Methods} & \textbf{MATH500} & \textbf{Economics} & \textbf{Engineering} & \textbf{TruthfulQA} & \textbf{Average} \\
    \midrule
    \multicolumn{6}{c}{\textit{Qwen2.5-7B-Instruct}} \\
    \midrule
    \rowcolor{cyan!15} \ourmethod{}  & \textbf{0.522} & \textbf{0.701} & \textbf{0.481} & \textbf{0.723} & \textbf{0.607} \\
    \quad - w/o Confidence  & 0.493 & 0.664 & 0.465 & 0.717 & 0.585 \\
    \quad - w/o Memory  & 0.493 & 0.664 & 0.453 & 0.693 & 0.576 \\
    \midrule
    \multicolumn{6}{c}{\textit{Qwen3-8B}} \\
    \midrule
    \rowcolor{cyan!15} \ourmethod{}  & \textbf{0.843} & \textbf{0.796} & 0.638 & \textbf{0.843} & 0.780 \\
    \quad - w/o Confidence  & 0.821 & 0.782 & 0.638 & 0.831 & 0.768 \\
    \quad - w/o Memory  & 0.836 & \textbf{0.796} & \textbf{0.654} & \textbf{0.843} & \textbf{0.782} \\
    \midrule
    \multicolumn{6}{c}{\textit{Gemma-3-4B-IT}} \\
    \midrule
    \rowcolor{cyan!15} \ourmethod{}  & 0.530 & \textbf{0.521} & \textbf{0.263} & \textbf{0.705} & \textbf{0.505} \\
    \quad - w/o Confidence  & 0.522 & 0.474 & \textbf{0.263} & 0.699 & 0.479 \\
    \quad - w/o Memory  & \textbf{0.537} & 0.502 & 0.239 & 0.657 & 0.484 \\
    \bottomrule
\end{tabular}
\vspace{-0.1cm}
\caption{Ablation study on \ourmethod{}. "w/o Confidence" removes confidence weights and annotations from peer responses; "w/o Memory" removes retrieved experiences from the task prompt.}
\label{tab:ablation}
\vspace{-0.4cm}
\end{table*}

\paragraph{RQ2: Do both components contribute to the improvement?}
To evaluate the individual contributions of the two core components, we conduct an ablation study by removing each module separately: (1) \textbf{\ourmethod{} w/o Confidence}, which uses memory retrieval but do not assign confidence weights to all agents' responses; and (2) \textbf{\ourmethod{} w/o Memory}, which disables adding retrieved experiences into the task prompt and relies solely on confidence weighting. Results are shown in \cref{tab:ablation}. On Qwen2.5-7B-Instruct, the full \ourmethod{} achieves an average accuracy of 0.607, while removing confidence weighting (w/o Confidence) drops performance to 0.585, and removing memory retrieval (w/o Memory) further drops it to 0.576. Both components contribute positively, and the full model consistently outperforms either ablated variant, confirming their complementary nature. Looking across models, confidence weighting shows a more consistent effect: on Gemma-3-4B-IT, removing confidence causes a notable drop in Economics and TruthfulQA, while removing memory leads to a larger drop in TruthfulQA. On Qwen3-8B, where the base MAD performance is already strong, the individual contributions are smaller, but the full model still achieves better performance. These results suggest that both mechanisms provide distinct benefits and that their relative importance varies with the model and task characteristics.

\begin{table}[t]
\centering
\small
\setlength{\tabcolsep}{3pt}
\vspace{0.1cm}
\begin{tabular}{l | c c | c c | c}
    \toprule
    \multirow{2}{*}{\textbf{Policy}} & \multicolumn{2}{c|}{\textit{Qwen2.5-7B}} & \multicolumn{2}{c|}{\textit{Gemma-3-4B}} & \multirow{2}{*}{\textbf{Avg.}} \\
    & \textbf{Econ.} & \textbf{TQA} & \textbf{Econ.} & \textbf{TQA} & \\
    \midrule
    Random                  & 0.668 & 0.705 & 0.496 & 0.671 & 0.635 \\
    Similarity-based        & 0.664 & 0.717 & 0.521 & 0.685 & 0.647 \\
    Positive-only           & 0.686 & \textbf{0.723} & 0.507 & \textbf{0.705} & 0.655 \\
    Diversity-only          & 0.682 & 0.717 & 0.512 & 0.685 & 0.649 \\
    Fixed-\(\lambda\) (0.25) & 0.676 & \textbf{0.723} & 0.523 & 0.695 & 0.654 \\
    Fixed-\(\lambda\) (0.50) & 0.686 & 0.711 & 0.504 & 0.697 & 0.650 \\
    Fixed-\(\lambda\) (0.75) & 0.684 & \textbf{0.723} & \textbf{0.528} & 0.691 & 0.657 \\
    \rowcolor{cyan!15} Debate-State-Aware & \textbf{0.701} & \textbf{0.723} & 0.521 & \textbf{0.705} & \textbf{0.663} \\
    \bottomrule
\end{tabular}
\vspace{-0.1cm}
\caption{Ablation over retrieval policies. \textbf{Bold} indicates the best result in each column.}
\label{tab:retrieval_ablation}
\vspace{-0.6cm}
\end{table}

\paragraph{RQ3: Does the debate-state-aware policy drive the gain?}
RQ2 establishes that memory helps, but not that our proposed way of retrieving it is responsible. To isolate the retrieval design, we hold every other component fixed and replace only the policy in Eq.~\eqref{eq:mmr} with five alternatives: \emph{random} selection from the memory bank, \emph{similarity-based} selection on task similarity alone, \emph{positive-only} selection restricted to cases with positive rewards, \emph{diversity-only} selection that drops the relevance term, and \emph{fixed-\(\lambda\)} variants that keep the MMR trade-off constant instead of conditioning it on the consensus ratio. \cref{tab:retrieval_ablation} reports the comparison on two models and two benchmarks.

Our policy attains the highest average accuracy (0.663), ahead of the best fixed-\(\lambda\) setting (0.657) and clearly ahead of random (0.635) and similarity-based (0.647) retrieval. That random retrieval is the weakest confirms that the gain does not follow from merely inserting past cases into the prompt, while similarity-based retrieval, the standard choice in memory-augmented agents, also trails every state-aware variant. Our policy is the best or tied-best in three of the four individual settings, demonstrating its effectiveness across different models and tasks. Taken together, these comparisons indicate that the improvement comes from how memory is selected rather than from its mere availability, and that letting the retrieval trade-off respond to the debate state is what makes the memory bank useful.

\begin{figure}[ht]
    \centering
    \begin{subfigure}{\linewidth}
        \centering
        \includegraphics[width=\linewidth]{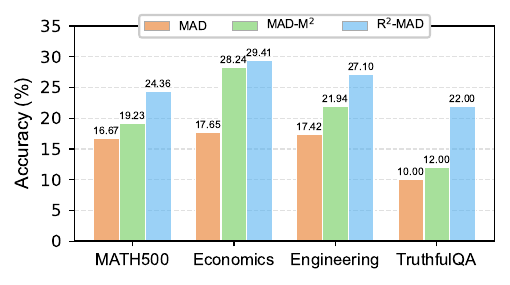}
        \vspace{-0.65cm}
        \caption{Qwen2.5-7B-Instruct}
        \label{fig:qwen2.5_eval}
    \end{subfigure}
    \vspace{-0.1cm}
    \begin{subfigure}{\linewidth}
        \centering
        \includegraphics[width=\linewidth]{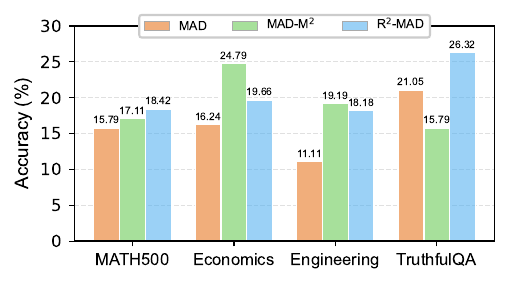}
        \vspace{-0.65cm}
        \caption{Gemma-3-4B-IT}
        \label{fig:gemma3_eval}
    \end{subfigure}
    \caption{Final accuracy of different debate methods with different LLMs on the shared misconception subset, where a majority of agents in standard MAD produced incorrect answers at round \(0\).}
    \label{fig:combined_eval}
    \vspace{-0.6cm}
\end{figure}

\paragraph{RQ4: Is \ourmethod{} effective under shared misconception?}
The overall results in RQ1 demonstrate the general effectiveness of \ourmethod{}, but do not reveal whether the gains come specifically from mitigating shared misconceptions, as our framework is designed to address. To directly evaluate this, we isolate the subset of instances where a majority of agents produced incorrect answers at round \(0\), which represents exactly the shared misconception regime. This is the most challenging setting for any debate-based method, as the peer skew actively works against the correct answer from the beginning. \cref{fig:combined_eval} presents the results on this subset. On Qwen2.5-7B-Instruct, \ourmethod{} substantially outperforms both MAD and \madmm{} across all four benchmarks, with the largest margins on Economics (29.41\% versus 17.65\% for MAD) and Engineering (27.10\% versus 17.42\%). A similar pattern holds on Gemma-3-4B-IT, where \ourmethod{} leads on three out of four benchmarks, with the most notable improvement on TruthfulQA. Importantly, the relative improvements on this subset are considerably larger than those observed in the overall results (\cref{tab:main_results}), confirming that the gains are concentrated in the regime \ourmethod{} targets.

\begin{table}[t]
\centering
\small
\setlength{\tabcolsep}{4pt}
\vspace{0.1cm}
\begin{tabular}{l l | c c}
    \toprule
    \textbf{Benchmark} & \textbf{Method} & \textbf{C}\(\rightarrow\)\textbf{W} \(\downarrow\) & \textbf{W}\(\rightarrow\)\textbf{C} \(\uparrow\) \\
    \midrule
    \multirow{2}{*}{MATH500} & \ourmethod{} & \textbf{0.222} & \textbf{0.127} \\
    & \quad - w/o Conf & 0.294 & 0.095 \\
    \midrule
    \multirow{2}{*}{Economics} & \ourmethod{} & 0.157 & \textbf{0.172} \\
    & \quad - w/o Conf & \textbf{0.148} & 0.075 \\
    \midrule
    \multirow{2}{*}{Engineering} & \ourmethod{} & \textbf{0.255} & 0.164 \\
    & \quad - w/o Conf & 0.410 & \textbf{0.173} \\
    \midrule
    \multirow{2}{*}{TruthfulQA} & \ourmethod{} & \textbf{0.208} & \textbf{0.103} \\
    & \quad - w/o Conf & 0.269 & 0.089 \\
    \bottomrule
\end{tabular}
\vspace{-0.1cm}
\caption{Stance-transition rates on the shared misconception subset with Qwen2.5-7B-Instruct. C \(\rightarrow\) W denotes correct agents switching to a wrong answer (lower is better) and W \(\rightarrow\) C the reverse (higher is better).}
\label{tab:flip_recovery}
\vspace{-0.6cm}
\end{table}

However, final accuracy alone does not show how the debate arrives at its answer, which is what \cref{prop:anti-majority} concerns: down-weighting an erroneous majority should make correct agents less likely to abandon their stance. To examine this, we further measure two stance-transition rates on the same subset: the fraction of agent-round transitions in which a correct agent switches to a wrong answer (C \(\rightarrow\) W) and the fraction in which a wrong agent recovers the correct one (W \(\rightarrow\) C), and compare \ourmethod{} against its w/o Confidence variant to demonstrate the effectiveness of confidence weighting in mitigating the \emph{majority dominance}. As illustrated in \cref{tab:flip_recovery}, confidence weighting successfully reduces harmful flips on three benchmarks, most notably on Engineering (0.410 to 0.255), and increases recoveries on three of the four, with Economics more than doubling (0.075 to 0.172). Corresponding results for the other two models are given in Appendix~\ref{appendix:flip_extra}.

\section{Conclusion}
In this paper, we propose \textbf{\ourmethod{}}, a framework that addresses the shared misconception problem in multi-agent debate by leveraging experience memory from past debates. Guided by the observation that shared misconceptions stem from two compounding factors in the debate process, a biased concept prior and an amplifying peer skew, \ourmethod{} introduces two complementary mechanisms: a debate-state-aware retrieval policy that calibrates agents' prior beliefs by dynamically selecting relevant historical evidence, and a memory-derived confidence weighting scheme that modulates peer influence based on estimated agent reliability. Experiments across various benchmarks and models demonstrate that \ourmethod{} consistently improves over both single-agent methods and MAD baselines, with particularly notable gains under the shared misconception regime. Ablation studies further confirm that both components contribute distinct benefits and that their combination yields the best overall performance.

\section*{Limitations}
The main limitation of \ourmethod{} is that it introduces additional computation on top of standard debate. Each round requires summarizing the current debate state, retrieving experiences, and estimating a confidence score for every peer. This overhead is a bounded constant factor rather than one that grows with task difficulty, yet it is a real cost, and \ourmethod{} is correspondingly less suitable than plain debate for latency-critical deployments.

The framework depends on the accumulated memory being informative for the tasks it is applied to. Building this memory requires an outcome signal for each past debate, which restricts our method to settings where such a signal is available. Currently, the memory is constructed offline and remains fixed at test time, which means its usefulness may degrade when the tasks encountered at deployment diverge substantially from those the memory was accumulated on. This dependence also carries a risk: since retrieved experiences shape agents' priors, a memory bank built from systematically mistaken trajectories could reinforce an erroneous consensus rather than correct it. So we recommend auditing the memory bank before applying the framework in consequential settings. And extending the framework to update memory continually during deployment, and to weaker forms of outcome supervision, is left to future work.

\section*{Acknowledgments}
Haifeng Zhang was supported in part by the National Natural Science Foundation of China under the Original Exploration Program (Grant No. 72450002).

\bibliography{references}


\appendix

\section{Additional Experimental Details}
\label{appendix:exp}

\subsection{Benchmark Details}
To comprehensively evaluate the performance of \ourmethod{}, we conduct experiments on four benchmarks that cover both reasoning and knowledge-intensive tasks, which are detailed as follows:
\begin{itemize}[leftmargin=12pt, topsep=5pt, itemsep=0pt]
    \item \textbf{MATH500}: This dataset is a subset of the MATH benchmark, consisting of 500 mathematical reasoning problems categorized by topic and difficulty \citep{lightman2024let}. To ensure a rigorously challenging evaluation, we selected 134 problems with the highest difficulty level (Level 5) from the MATH500 dataset as the test set, while adopting the remaining as the training set for memory construction.
    
    \item \textbf{MMLU-Pro Economics}: MMLU-Pro is an advanced benchmark assessing multi-disciplinary language understanding and reasoning across 14 domains \citep{wang2024mmlu}.
    In this work, we leveraged the Economics subset, which consists of 844 questions with an expanded 10-option multiple-choice format. During the experiment, the dataset was randomly divided into training and test sets in a 3:1 ratio. Consequently, 633 questions were allocated for experience accumulation and 211 for evaluation.
    
    \item \textbf{MMLU-Pro Engineering}: In order to measure the effectiveness of our method in more domains, we also employed the Engineering subset of MMLU-Pro \citep{wang2024mmlu}. This subset contains 969 challenging questions with a 10-option multiple-choice format. Following the same 3:1 split strategy, we randomly partitioned the dataset, resulting in 727 questions for the training set and 242 questions for the test set.
    
    \item \textbf{TruthfulQA}: This is a benchmark designed to evaluate whether language models generate truthful answers to questions that are adversarially crafted to elicit common misconceptions and falsehoods \citep{lin2022truthfulqa}. To facilitate efficient and automated correctness verification, we adopted the multiple-choice evaluation format provided by the dataset, in which the model selects the correct answer from a set of candidate options. However, TruthfulQA's multiple-choice variant features a variable number of options per question. To maintain a sufficiently challenging evaluation setting, we filtered for questions with 4 to 9 candidate options, yielding a total of 664 questions. Following the same 3:1 random split strategy, 498 questions were allocated to the training set and 166 to the test set.
\end{itemize}

\begin{table*}[!ht]
\centering
\begin{tabular}{l | c c c c | c}
    \toprule
    \textbf{Methods} & \textbf{MATH500} & \textbf{Economics} & \textbf{Engineering} & \textbf{TruthfulQA} & \textbf{Average} \\
    \midrule
    \multicolumn{6}{c}{\textit{Qwen2.5-7B-Instruct}} \\
    \midrule
    \madmm{} (S)     & 0.410 & 0.597 & 0.346 & 0.560 & 0.478 \\
    \madmm{} (O)     & 0.403 & 0.635 & 0.374 & 0.590 & 0.501 \\
    \midrule
    \multicolumn{6}{c}{\textit{Qwen3-8B}} \\
    \midrule
    \madmm{} (S)    & 0.836 & 0.796 & 0.700 & 0.711 & 0.761 \\
    \madmm{} (O)    & 0.769 & 0.787 & 0.667 & 0.705 & 0.732 \\
    \midrule
    \multicolumn{6}{c}{\textit{Gemma-3-4B-IT}} \\
    \midrule
    \madmm{} (S)    & 0.448 & 0.550 & 0.247 & 0.428 & 0.418 \\
    \madmm{} (O)    & 0.537 & 0.507 & 0.259 & 0.448 & 0.448 \\
    \bottomrule
\end{tabular}
\vspace{-0.1cm}
\caption{Comparison of \madmm{} under subjective (S) and objective (O) masking strategies across four benchmarks and three large language models.}
\label{tab:mad-m2}
\vspace{-0.4cm}
\end{table*}

\subsection{Baseline Details}
To validate the effectiveness of our proposed \ourmethod{}, we adopted the following baselines:
\begin{itemize}[leftmargin=12pt, topsep=5pt, itemsep=0pt]
    \item \textbf{Chain of Thought (CoT)}: Chain of Thought prompting enhances the reasoning capability of LLMs by instructing the model to decompose a complex problem into intermediate reasoning steps before arriving at the final answer \citep{wei2022chain}. In our experiments, each agent generates a single response using CoT prompting, and the answer is directly extracted from this response.

    \item \textbf{Self-Consistency}: Self-Consistency improves upon CoT by sampling multiple independent reasoning paths for a given question and aggregating final answers by marginalizing out the reasoning paths \citep{wang2022self}. In our settings, the final answer is selected through majority voting over the sampled responses. And to ensure a fair comparison with multi-agent debate methods, which employ 3 agents over 3 debate rounds, we sampled 9 independent reasoning paths for Self-Consistency. Appendix~\ref{appendix:sc_paths} reports how the number of reasoning paths affects its final accuracy.

    \item \textbf{Multi-Agent Debate (MAD)}: Multi-Agent Debate employs multiple LLM agents to iteratively refine their responses through multi-round discussion \citep{du2024improving}. At the initial round, each agent independently generates a response. In subsequent rounds, each agent observes all peers' responses from the previous round and updates its answer accordingly. The final answer is also determined by majority voting over the last-round responses. Following the default configuration, we set the number of agents to 3 and the number of debate rounds to 3.

    \item \textbf{\madmm{}}: Multi-Agent Debate with Memory Masking addresses the vulnerability of the standard MAD framework to erroneous memories by introducing an evaluation-and-masking phase between debate rounds \citep{tian2026multi}. Before each subsequent round, agents critically evaluate the responses from the previous round and generate a binary mask to filter out potentially incorrect memories. \madmm{} supports both a subjective masking, where agents explicitly judge each response, and an objective masking strategy based on response perplexity. In our main experiments, we adopt the objective masking variant with the default configuration of 3 agents and 3 debate rounds.
    
    \item \textbf{ICL-CoT}: To test whether the benefit of \ourmethod{} can be attributed simply to having access to related past cases, we construct a retrieval-augmented single-agent baseline. For each test question, it retrieves cases from exactly the same memory bank that \ourmethod{} uses and places them in the prompt as in-context examples before performing standard CoT reasoning. Since a single agent has no debate state, retrieval here is keyed on task similarity. This baseline therefore receives the same memory content as \ourmethod{} but none of its debate-state-aware retrieval or confidence weighting.
\end{itemize}

\begin{table*}[t]
\centering
\begin{tabular}{l | c c c c | c}
    \toprule
    \textbf{Methods} & \textbf{MATH500} & \textbf{Economics} & \textbf{Engineering} & \textbf{TruthfulQA} & \textbf{Average} \\
    \midrule
    \multicolumn{6}{c}{\textit{Qwen2.5-7B-Instruct}} \\
    \midrule
    Self-Consistency (3)   & \textbf{0.530} & 0.644 & 0.457 & 0.686 & 0.579 \\
    Self-Consistency (6)   & 0.507 & 0.649 & 0.469 & \textbf{0.704} & 0.582 \\
    Self-Consistency (9)   & 0.522  & \textbf{0.668} & \textbf{0.477} & \textbf{0.704} & \textbf{0.593} \\
    \midrule
    \multicolumn{6}{c}{\textit{Qwen3-8B}} \\
    \midrule
    Self-Consistency (3)   & 0.649 & 0.768 & 0.490 & 0.819 & 0.682 \\
    Self-Consistency (6)   & 0.649 & 0.787 & 0.490 & \textbf{0.831} & 0.689 \\
    Self-Consistency (9)   & \textbf{0.671} & \textbf{0.791}  & \textbf{0.535} & 0.825 & \textbf{0.706} \\
    \midrule
    \multicolumn{6}{c}{\textit{Gemma-3-4B-IT}} \\
    \midrule
    Self-Consistency (3)   & 0.485 & \textbf{0.498} & 0.222 & \textbf{0.657} & 0.466 \\
    Self-Consistency (6)   & 0.522 & 0.479 & \textbf{0.259} & 0.608  & 0.467 \\
    Self-Consistency (9)   & \textbf{0.530} & 0.488  & 0.255 & 0.632 & \textbf{0.476} \\
    \bottomrule
\end{tabular}
\vspace{-0.1cm}
\caption{Accuracy of self-consistency with different numbers of paths across four benchmarks and three models. \textbf{Bold} values indicate the best result in each model's column.}
\label{tab:sc_ablation}
\vspace{-0.4cm}
\end{table*}

\subsection{Implementation Details}
\label{appendix:impl_details}

\paragraph{Model Configuration.}
Experiments are mainly conducted using three open-source LLMs: Qwen2.5-7B-Instruct \citep{yang2024qwen2}, Qwen3-8B \citep{yang2025qwen3}, and Gemma-3-4B-IT \citep{gemma2025gemma3}. We access all models via local inference with vLLM \citep{kwon2023efficient}. To encourage diverse responses, the generation temperature is set to \(1.0\) for all models. Moreover, the \texttt{top\_p} parameter is set to \(1.0\) across all models. And the maximum output token length for each response is set to \(6144\) for all models during both the training and evaluation stages. Other LLM-related hyperparameters are set to their default values. For the two larger models, Llama-3.3-70B-Instruct \citep{grattafiori2024llama} and GPT-4o-mini \citep{openai2024gpt4}, we use API access instead of local inference, keeping the same 3-agent, 3-round configuration, the same agent personas and prompts, and the same decoding settings where the API exposes them.

\paragraph{Memory Construction.}
The experience memory is constructed offline before evaluation by running the full debate procedure on the training set of each benchmark. Specifically, for each question, we run a standard 3-agent, 3-round debate using the agent personas and prompts described in Appendix~\ref{appendix:prompts}. To enrich the diversity of collected experiences, we force all agents to continue debating until the maximum number of rounds is reached, regardless of whether consensus has been established. This ensures the memory bank captures a wider range of debate dynamics, including scenarios where agents must break out of a premature consensus. Since the initial round (\(t = 0\)) involves each agent answering independently without access to peer responses or debate state information, we exclude round-\(0\) responses from the memory. For each subsequent round (\(t > 0\)), we extract a memory case following the structure defined in Eq.\eqref{eq:memory_case} and store it in the corresponding agent's memory bank. In our settings, each agent maintains its own memory bank, reflecting debate trajectories and outcomes from its own perspective. The resulting memory bank for each agent contains approximately \(2N_\text{train}\) cases per benchmark, where \(N_\text{train}\) is the number of questions in the training set of the benchmark.

\paragraph{Debate and Retrieval Procedure.}
At test time, each debate runs for 3 rounds with 3 agents. At each round \(t > 0\), the retrieval policy first encodes the current debate state into an embedding vector using BGE-M3 \citep{chen2024m3}. A candidate set of \(3K\) memory cases is then retrieved by ranking all entries in the agent's memory bank by cosine similarity to this embedding and retaining the top-\(3K\). This initial filtering serves two purposes: it reduces the computational cost of the subsequent MMR selection, and it establishes a lower bound ensuring that all final candidates maintain a minimum level of relevance to the current debate state. From this candidate set, \(K\) cases are selected via MMR with the consensus-based trade-off coefficient \(\lambda_t = 1 - \gamma \cdot \text{Cons}(Z^{(t-1)})\), where \(\gamma = 0.9\). For confidence weighting, the thresholds are set to \(w_h = 0.55\) and \(w_l = 0.45\) in practice. And when \(c_{j,i}^{(t)} < w_l\) or \(c_{j,i}^{(t)} > w_h\), agent \(j\)'s response will be appended with an annotation enclosed by \texttt{<confidence></confidence>}; otherwise, no annotation is added. The consensus ratio is computed by extracting answers from each agent's response and calculating the fraction of agents agreeing on the most frequent answer. During evaluation, early stopping is applied if all agents reach consensus before round \(T\). For the final answer, majority voting is performed over the responses of all agents.

\begin{table*}[t]
\centering
\setlength{\tabcolsep}{6pt}
\begin{tabular}{l | c c c c | c}
    \toprule
    \textbf{Methods} & \textbf{MATH500} & \textbf{Economics} & \textbf{Engineering} & \textbf{TruthfulQA} & \textbf{Average} \\
    \midrule
    \multicolumn{6}{c}{\textit{Llama-3.3-70B-Instruct}} \\
    \midrule
    CoT          & 0.485 & 0.754 & 0.539 & 0.799 & 0.644 \\
    MAD          & 0.542 & 0.750 & 0.642 & \textbf{0.874} & 0.702 \\
    \rowcolor{cyan!15} \ourmethod{} & \textbf{0.560} & \textbf{0.766} & \textbf{0.667} & \textbf{0.874} & \textbf{0.717} \\
    \midrule
    \multicolumn{6}{c}{\textit{GPT-4o-mini}} \\
    \midrule
    CoT          & 0.500 & 0.703 & 0.391 & 0.829 & 0.606 \\
    MAD          & \textbf{0.560} & 0.728 & 0.429 & \textbf{0.874} & 0.648 \\
    \rowcolor{cyan!15} \ourmethod{} & \textbf{0.560} & \textbf{0.735} & \textbf{0.486} & 0.861 & \textbf{0.661} \\
    \bottomrule
\end{tabular}
\vspace{-0.1cm}
\caption{Overall accuracy on two larger models: Llama-3.3-70B-Instruct and GPT-4o-mini. \textbf{Bold} values indicate the best result in each model's column.}
\label{tab:large_models}
\vspace{-0.4cm}
\end{table*}

\section{Additional Experimental Results}
\label{appendix:additional_exp}

\subsection{Comparison of \madmm{} Masking Strategies}
In the main experiments (\cref{tab:main_results}), we adopt the objective (O) masking variant of \madmm{} as the default, since it achieves higher average accuracy across most models. Here, we report the full comparison between the two masking strategies for completeness. \cref{tab:mad-m2} presents the results of \madmm{} under both the subjective (S) masking strategy, where agents explicitly evaluate each peer response and generate binary masks, and the objective (O) masking strategy, which retains only the response with the lowest perplexity. The results reveal that neither strategy consistently dominates across all settings: their relative effectiveness varies with the model and task, consistent with findings in \citet{tian2026multi}. Importantly, \ourmethod{} outperforms both variants across all three models in average accuracy, suggesting that leveraging retrieved experiences for prior correction and confidence weighting provides more robust improvement.

\subsection{Number of Reasoning Paths in Self-Consistency}
\label{appendix:sc_paths}
The compute available to Self-Consistency is set by a free parameter, the number of sampled reasoning paths, so how that parameter is chosen determines how strong a baseline it is. To keep the comparison in \cref{tab:main_results} fair, we set it by matching compute: since all debate-based methods in this work use 3 agents over 3 rounds, we sample 9 reasoning paths so that Self-Consistency issues the same number of generation calls per query. \cref{tab:sc_ablation} reports what this choice costs us by evaluating 3 and 6 paths in addition. Average accuracy increases monotonically with the number of paths on all three models, so the setting we adopt is also the strongest of the three for Self-Consistency on every model. The comparison in \cref{tab:main_results} is therefore against the best-performing configuration of this baseline rather than a weakened one, and \ourmethod{} still attains higher average accuracy on all three models.

\subsection{Additional Shared Misconception Results}

\begin{figure}[ht]
    \centering
    \vspace{-0.3cm}
    \includegraphics[width=\linewidth]{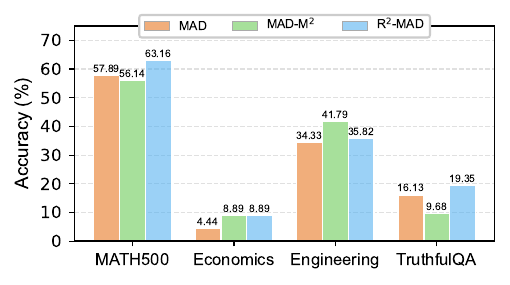}
    \vspace{-0.75cm}
    \caption{Final accuracy of different debate methods with Qwen3-8B on the shared misconception subset, where a majority of agents in standard MAD produced incorrect answers at round \(0\).}
    \label{fig:qwen3_eval}
    \vspace{-0.3cm}
\end{figure}

In \cref{sec:exp_results} (RQ4), we present the shared misconception analysis for Qwen2.5-7B-Instruct and Gemma-3-4B-IT in \cref{fig:combined_eval}. Here we provide the corresponding results for Qwen3-8B in \cref{fig:qwen3_eval}. On this stronger model, \ourmethod{} remains effective under the shared misconception regime, leading on three benchmarks. It is interesting that in Economics, all three debate methods struggle with accuracies below 10\%, suggesting that when a strong model converges on misconceptions at prior, the errors are particularly resistant to correction. Overall, the Qwen3-8B results reinforce the pattern observed across the other two models: \ourmethod{}'s combination of memory-based prior correction and confidence weighting is effective at mitigating shared misconceptions across models of varying capability.

\subsection{Additional Flip and Recovery Analysis}
\label{appendix:flip_extra}

\begin{table}[ht]
\centering
\small
\setlength{\tabcolsep}{4pt}
\begin{tabular}{l l | c c}
    \toprule
    \textbf{Benchmark} & \textbf{Method} & \textbf{C}\(\rightarrow\)\textbf{W} \(\downarrow\) & \textbf{W}\(\rightarrow\)\textbf{C} \(\uparrow\) \\
    \midrule
    \multicolumn{4}{c}{\textit{Qwen3-8B}} \\
    \midrule
    \multirow{2}{*}{MATH500}     & \ourmethod{} & \textbf{0.032} & \textbf{0.557} \\
                                 & \quad - w/o Conf & 0.038 & 0.524 \\
    \multirow{2}{*}{Economics}   & \ourmethod{} & \textbf{0.077} & \textbf{0.049} \\
                                 & \quad - w/o Conf & 0.333 & 0.033 \\
    \multirow{2}{*}{Engineering} & \ourmethod{} & \textbf{0.180} & \textbf{0.290} \\
                                 & \quad - w/o Conf & 0.189 & 0.288 \\
    \multirow{2}{*}{TruthfulQA}  & \ourmethod{} & 0.294 & 0.079 \\
                                 & \quad - w/o Conf & \textbf{0.222} & \textbf{0.133} \\
    \midrule
    \multicolumn{4}{c}{\textit{Gemma-3-4B-IT}} \\
    \midrule
    \multirow{2}{*}{MATH500}     & \ourmethod{} & \textbf{0.348} & 0.115 \\
                                 & \quad - w/o Conf & 0.393 & \textbf{0.117} \\
    \multirow{2}{*}{Economics}   & \ourmethod{} & \textbf{0.196} & \textbf{0.098} \\
                                 & \quad - w/o Conf & 0.385 & 0.077 \\
    \multirow{2}{*}{Engineering} & \ourmethod{} & 0.461 & \textbf{0.127} \\
                                 & \quad - w/o Conf & \textbf{0.402} & 0.115 \\
    \multirow{2}{*}{TruthfulQA}  & \ourmethod{} & \textbf{0.182} & \textbf{0.142} \\
                                 & \quad - w/o Conf & 0.233 & 0.094 \\
    \bottomrule
\end{tabular}
\vspace{-0.1cm}
\caption{Stance-transition rates on the shared misconception subset for Qwen3-8B and Gemma-3-4B-IT. C \(\rightarrow\) W denotes correct agents switching to a wrong answer and W \(\rightarrow\) C the reverse.}
\label{tab:flip_recovery_extra}
\vspace{-0.5cm}
\end{table}

Here we provide the corresponding results for the other two models in \cref{tab:flip_recovery_extra}, measured on the same shared misconception subset and under the same definitions. The pattern observed in \cref{tab:flip_recovery} still holds for both models. For each of them, confidence weighting lowers the rate at which correct agents abandon their stance on three of the four benchmarks and raises the recovery rate in most cases. The effect is most pronounced on Economics, where the flip rate falls from 0.333 to 0.077 on Qwen3-8B and from 0.385 to 0.196 on Gemma-3-4B-IT. Counting across all three models, each of the two rates improves in nine of the twelve model-benchmark pairs, and the largest differences in the table are consistently in favor of the full model. These stance-transition rates isolate the contribution of confidence weighting, and they move in the direction \cref{prop:anti-majority} predicts, consistent with the gains \ourmethod{} obtains over the debate baselines on the shared misconception subset.

\subsection{Comparison with a Retrieval-Augmented Single-Agent Baseline}
\label{appendix:icl_cot}

\begin{table}[ht]
\centering
\small
\setlength{\tabcolsep}{3pt}
\vspace{-0.1cm}
\begin{tabular}{l l | c c c}
    \toprule
    \textbf{Model} & \textbf{Benchmark} & \textbf{CoT} & \textbf{ICL-CoT} & \textbf{\ourmethod{}} \\
    \midrule
    \multirow{2}{*}{Qwen2.5-7B} & Economics & 0.649 & 0.638 & \textbf{0.701} \\
                                & TruthfulQA   & 0.675 & \textbf{0.723} & \textbf{0.723} \\
    \midrule
    \multirow{2}{*}{Gemma-3-4B} & Economics & 0.446 & 0.501 & \textbf{0.521} \\
                                & TruthfulQA   & 0.627 & 0.619 & \textbf{0.705} \\
    \midrule
    \multicolumn{2}{l|}{\textit{Average}} & 0.599 & 0.620 & \textbf{0.663} \\
    \bottomrule
\end{tabular}
\vspace{-0.1cm}
\caption{Comparison against ICL-CoT, a single-agent baseline that retrieves from the same memory bank and leverages the retrieved cases as in-context examples.}
\label{tab:icl_cot}
\vspace{-0.3cm}
\end{table}

A natural concern about any memory-augmented method is that its advantage may come from nothing more than access to related question-answer pairs, in which case a single agent given the same cases should perform comparably. We test this directly with ICL-CoT, which retrieves from the identical memory bank and receives the same number of cases in its prompt, but performs single-agent CoT reasoning rather than debate.

The comparison results are listed in \cref{tab:icl_cot}. \ourmethod{} matches or exceeds ICL-CoT in all four settings and is strictly better in three, showing that the same memory content yields a larger benefit when it is retrieved according to the debate state and used to reweight peer influence than when it is inserted as static exemplars. Indeed, ICL-CoT does not even improve reliably over naive CoT, falling below it on Economics with Qwen2.5-7B-Instruct and on TruthfulQA with Gemma-3-4B-IT, which indicates that having related cases available is not by itself sufficient. The advantage of \ourmethod{} therefore does not reduce to access to similar question-answer pairs.

\subsection{Results on Larger Models}
\label{appendix:large_models}

The three models used in \cref{sec:exp} all fall in the 4B--8B range. To test whether the benefit of \ourmethod{} persists on stronger models, we further evaluate \ourmethod{} against CoT and MAD on Llama-3.3-70B-Instruct \citep{grattafiori2024llama} and GPT-4o-mini \citep{openai2024gpt4}, covering both open-source and commercial settings. As shown in \cref{tab:large_models}, \ourmethod{} attains the best average accuracy on both models, improving over MAD from 0.702 to 0.717 on Llama-3.3-70B-Instruct and from 0.648 to 0.661 on GPT-4o-mini, which indicates that the benefit is not confined to the small-model regime. We note, however, that the per-benchmark picture is more mixed than in \cref{tab:main_results}: the gains concentrate on Economics and Engineering, while TruthfulQA is matched on Llama-3.3-70B-Instruct and slightly lower than MAD on GPT-4o-mini. A plausible reading is that stronger LLMs already resolve much of what memory would otherwise supply on the more knowledge-saturated tasks, leaving less headroom for prior correction.

\subsection{Computational Cost}
\label{appendix:cost}

\begin{table}[ht]
\centering
\small
\setlength{\tabcolsep}{3pt}
\vspace{-0.1cm}
\begin{tabular}{l l | c c c}
    \toprule
    \textbf{Model} & \textbf{Method} & \makecell{\textbf{Debate}\\\textbf{Tok. (\#)}} & \makecell{\textbf{Summary}\\\textbf{Tok. (\#)}} & \makecell{\textbf{Time}\\\textbf{(s)}} \\
    \midrule
    \multirow{3}{*}{Qwen2.5-7B} & SC     & 5.75K & --    & 0.62 \\
                                & MAD    & 4.44K & --    & 0.49 \\
                                & \ourmethod{} & 5.41K & 2.25K & 0.96 \\
    \midrule
    \multirow{3}{*}{Qwen3-8B}   & SC     & 14.67K & --   & 5.65 \\
                                & MAD    & 6.16K & --    & 3.75 \\
                                & \ourmethod{} & 6.38K & 1.85K & 6.28 \\
    \midrule
    \multirow{3}{*}{Gemma-3-4B} & SC     & 6.09K & --    & 0.80 \\
                                & MAD    & 6.99K & --    & 0.81 \\
                                & \ourmethod{} & 7.87K & 2.76K & 1.18 \\
    \bottomrule
\end{tabular}
\vspace{-0.1cm}
\caption{Per-query token consumption and wall-clock time statistics on Economics. SC uses 9 reasoning paths, matching the configuration reported in \cref{tab:main_results}.}
\label{tab:cost}
\vspace{-0.3cm}
\end{table}

We characterize the per-query cost of \ourmethod{} along three sources. Debate and summarization invoke the LLMs and are therefore measured in tokens, whereas retrieval and confidence estimation are purely embedding-based, issue no generation call, so are approximately measured by wall-clock time. \cref{tab:cost} reports both on Economics, with Self-Consistency (SC) and MAD as references.

The debate token consumption of \ourmethod{} exceeds MAD by roughly 0.2K--1.0K tokens per query, which reflects the fixed number of retrieved memory tokens injected per round; because \(K\) is constant, this component does not grow with problem scale. Summarization adds a further 1.9K--2.8K tokens and is the only extra generation call our framework introduces. In wall-clock terms, \ourmethod{} is roughly \(1.5\times\) to \(2\times\) MAD across the three models and modestly slower than Self-Consistency, so all the methods compared remain within the same order of magnitude. Therefore, the overhead is a bounded constant factor rather than one that grows with task difficulty.

\newpage
\onecolumn

\section{Proofs}
\label{appendix:proof}

\subsection{Proof for Proposition \ref{prop:prior-correction}}
\label{proof:prior-correction}
\begin{proof}
    \textbf{Part 1: Prior correction identity.}
    Applying Bayes' rule to the joint distribution of $(\theta, E_i^{(t)})$ conditional on $\phi_i$:
    \begin{equation}
        \mathbb{P}(\theta \mid E_i^{(t)}, \phi_i) = \frac{\mathbb{P}(E_i^{(t)} \mid \theta, \phi_i) \, \mathbb{P}(\theta \mid \phi_i)}{\mathbb{P}(E_i^{(t)} \mid \phi_i)},
    \end{equation}
    which is exactly the claimed identity. Here the denominator is the marginal likelihood $\mathbb{P}(E_i^{(t)} \mid \phi_i) = \sum_{\theta^\prime \in \Theta} \mathbb{P}(E_i^{(t)} \mid \theta^\prime, \phi_i) \, \mathbb{P}(\theta^\prime \mid \phi_i)$, which is positive and independent of $\theta$, so the identity is well-defined.
    
    \textbf{Part 2: Misconception suppression.}
    Taking the ratio of the corrected priors at $\theta^\star$ and $\theta^\prime$:
    \begin{equation}
        \frac{\mathbb{P}(\theta^\star \mid E_i^{(t)}, \phi_i)}{\mathbb{P}(\theta^\prime \mid E_i^{(t)}, \phi_i)} = 
        \frac{\mathbb{P}(\theta^\star \mid \phi_i) \cdot \dfrac{\mathbb{P}(E_i^{(t)} \mid \theta^\star, \phi_i)}{\mathbb{P}(E_i^{(t)} \mid \phi_i)}}{\mathbb{P}(\theta^\prime \mid \phi_i) \cdot \dfrac{\mathbb{P}(E_i^{(t)} \mid \theta^\prime, \phi_i)}{\mathbb{P}(E_i^{(t)} \mid \phi_i)}} = 
        \frac{\mathbb{P}(\theta^\star \mid \phi_i)}{\mathbb{P}(\theta^\prime \mid \phi_i)} \cdot \frac{\mathbb{P}(E_i^{(t)} \mid \theta^\star, \phi_i)}{\mathbb{P}(E_i^{(t)} \mid \theta^\prime, \phi_i)},
    \end{equation}
    where the $\mathbb{P}(E_i^{(t)} \mid \phi_i)$ terms cancel. The first factor is the bare prior ratio. The second factor is the likelihood ratio of the retrieved experiences under $\theta^\star$ versus $\theta^\prime$. By assumption, $\mathbb{P}(E_i^{(t)} \mid \theta^\star, \phi_i) > \mathbb{P}(E_i^{(t)} \mid \theta^\prime, \phi_i)$, so this likelihood ratio is strictly greater than $1$. Hence
    \begin{equation}
        \frac{\mathbb{P}(\theta^\star \mid E_i^{(t)}, \phi_i)}{\mathbb{P}(\theta^\prime \mid E_i^{(t)}, \phi_i)} > \frac{\mathbb{P}(\theta^\star \mid \phi_i)}{\mathbb{P}(\theta^\prime \mid \phi_i)},
    \end{equation}
    confirming that informative memory retrieval shifts the prior ratio in favor of the true concept $\theta^\star$, directly counteracting pre-existing bias toward $\theta^\prime$.
\end{proof}

\subsection{Proof for Proposition \ref{prop:anti-majority}}
\label{proof:anti-majority}
\begin{proof}
    Following the setup of Theorem 5.2 in \citet{estornell2024multi}, suppose $Z^{(t)}$ contains $m$ responses sharing a most-likely concept $\theta^\prime$, i.e., for $j \le m$, $\theta^\prime = \arg\max_\theta \mathbb{P}(z_j^{(t)} \mid \theta, \phi_i)$. Let $z'^{(t)}$ denote a canonical representative of any such majority response. For any two candidate responses $z_{(i,1)}^{(t+1)}$ and $z_{(i,2)}^{(t+1)}$, the ratio of their generation probabilities under confidence weighting is
    \begin{equation}\label{eq:anti-maj-ratio}
        \frac{\mathbb{P}_w(z_{(i,1)}^{(t+1)} \mid Z^{(t)}, x, \phi_i)}{\mathbb{P}_w(z_{(i,2)}^{(t+1)} \mid Z^{(t)}, x, \phi_i)} = 
        \frac{\displaystyle\sum_{\theta \in \Theta} \mathbb{P}(z_{(i,1)}^{(t+1)} \mid \theta, \phi_i)\, \mathbb{P}(x \mid \theta, \phi_i)\, \mathbb{P}(\theta \mid \phi_i) \prod_{j=1}^{n} \mathbb{P}(z_j^{(t)} \mid \theta, \phi_i)^{w_{j,i}^{(t)}}}{\displaystyle\sum_{\theta \in \Theta} \mathbb{P}(z_{(i,2)}^{(t+1)} \mid \theta, \phi_i)\, \mathbb{P}(x \mid \theta, \phi_i)\, \mathbb{P}(\theta \mid \phi_i) \prod_{j=1}^{n} \mathbb{P}(z_j^{(t)} \mid \theta, \phi_i)^{w_{j,i}^{(t)}}}.
    \end{equation}
    
    \textbf{Step 1: Separating majority and minority contributions.}
    With $w_{j,i}^{(t)} = \alpha$ for $j \le m$ and $w_{j,i}^{(t)} = 1$ for $j > m$, we split the product over agents:
    \begin{equation}
        \prod_{j=1}^{n} \mathbb{P}(z_j^{(t)} \mid \theta, \phi_i)^{w_{j,i}^{(t)}}
        \;=\;
        \underbrace{\prod_{j \le m} \mathbb{P}(z_j^{(t)} \mid \theta, \phi_i)^{\alpha}}_{\text{majority (weighted)}}
        \;\cdot\;
        \underbrace{\prod_{j > m} \mathbb{P}(z_j^{(t)} \mid \theta, \phi_i)}_{\text{minority (unweighted)}}.
    \end{equation}
    Since all $m$ majority responses share the most-likely concept $\theta^\prime$, we approximate $\prod_{j \le m} \mathbb{P}(z_j^{(t)} \mid \theta, \phi_i)^{\alpha} \approx \mathbb{P}(z'^{(t)} \mid \theta, \phi_i)^{\alpha m}$. Substituting into \eqref{eq:anti-maj-ratio}, both numerator and denominator take the form
    \begin{equation}
        \sum_{\theta \in \Theta} \mathbb{P}(z_{(i,\cdot)}^{(t+1)} \mid \theta, \phi_i)\, \mathbb{P}(x \mid \theta, \phi_i)\, \mathbb{P}(\theta \mid \phi_i) \prod_{j > m} \mathbb{P}(z_j^{(t)} \mid \theta, \phi_i) \;\cdot\; \mathbb{P}(z'^{(t)} \mid \theta, \phi_i)^{\alpha m}.
    \end{equation}
    
    \textbf{Step 2: Normalizing by the dominant term.}
    Dividing both numerator and denominator by $\mathbb{P}(z'^{(t)} \mid \theta^\prime, \phi_i)^{\alpha m}$, each summand corresponding to concept $\theta$ acquires the factor
    \begin{equation}
        \left(\frac{\mathbb{P}(z'^{(t)} \mid \theta, \phi_i)}{\mathbb{P}(z'^{(t)} \mid \theta^\prime, \phi_i)}\right)^{\!\alpha m}.
    \end{equation}
    By definition of $\theta^\prime$ as the concept maximizing $\mathbb{P}(z'^{(t)} \mid \theta, \phi_i)$, for every $\theta \neq \theta^\prime$ we have
    \begin{equation}
        \frac{\mathbb{P}(z'^{(t)} \mid \theta, \phi_i)}{\mathbb{P}(z'^{(t)} \mid \theta^\prime, \phi_i)} \;<\; 1.
    \end{equation}
    
    \textbf{Step 3: Taking $m \to \infty$.}
    For each $\theta \neq \theta^\prime$, the factor $\left(\mathbb{P}(z'^{(t)} \mid \theta, \phi_i) / \mathbb{P}(z'^{(t)} \mid \theta^\prime, \phi_i)\right)^{\alpha m}$ converges to $0$ as $m \to \infty$, since the base is strictly less than $1$ and the exponent $\alpha m \to \infty$. Only the $\theta = \theta^\prime$ summand (whose factor equals $1$) survives. Therefore
    \begin{equation}
        \lim_{m \to \infty} \frac{\mathbb{P}_w(z_{(i,1)}^{(t+1)} \mid Z^{(t)}, x, \phi_i)}{\mathbb{P}_w(z_{(i,2)}^{(t+1)} \mid Z^{(t)}, x, \phi_i)}
        \;=\;
        \frac{\mathbb{P}(z_{(i,1)}^{(t+1)} \mid \theta^\prime, \phi_i)}{\mathbb{P}(z_{(i,2)}^{(t+1)} \mid \theta^\prime, \phi_i)}.
    \end{equation}
    
    \textbf{Step 4: Convergence rate analysis.}
    The rate at which the non-$\theta^\prime$ terms vanish is governed by the exponent $\alpha m$. Specifically, for the leading competing concept $\theta^* \neq \theta^\prime$, define
    \begin{equation}
        r \;:=\; \frac{\mathbb{P}(z'^{(t)} \mid \theta^*, \phi_i)}{\mathbb{P}(z'^{(t)} \mid \theta^\prime, \phi_i)} \;\in\; (0, 1).
    \end{equation}
    Under standard debate ($w_{j,i}^{(t)} \equiv 1$), the ratio of the $\theta^*$ summand to the $\theta^\prime$ summand decays as $r^m$, giving a convergence rate of $O(m)$ in the log-scale (i.e., $\log r^m = m \log r$). Under confidence weighting with $\alpha < 1$, this decay becomes $r^{\alpha m}$, with log-rate $\alpha m \log r$. The convergence to $\theta^\prime$-dominance is therefore slowed from $O(m)$ to $O(\alpha m)$.
    
    Equivalently, to achieve the same degree of $\theta^\prime$-dominance that standard debate achieves at majority size $m$, the confidence-weighted debate requires an effective majority size of $m_{\mathrm{eff}} = m / \alpha > m$. Viewed from the other direction: a majority of size $m$ under confidence weighting has the same effect as a majority of size $\alpha m < m$ under standard debate. This completes the proof.
\end{proof}

\subsection{Joint Guarantee: Memory Lift and Confidence Lift}
\label{proof:joint}

We first establish that the two mechanisms contribute additively in log-odds space, and then derive a joint improvement bound.
\begin{definition}[Log-odds]
\label{def:log_odds}
    For agent $i$ at round $t$, define the log-odds between the true concept $\theta^\star$ and an erroneous concept $\theta^\prime$ under the joint model (Eq.~\eqref{eq:joint_gen}) as
    \begin{equation}
        L_i^{(t)}(w, E) := \log \frac{\mathbb{P}_{w,E}(\theta^\star \mid Z^{(t)}, x, E_i^{(t)}, \phi_i)}{\mathbb{P}_{w,E}(\theta^\prime \mid Z^{(t)}, x, E_i^{(t)}, \phi_i)}.
    \end{equation}
\end{definition}

\begin{lemma}[Log-odds decomposition]
\label{lemma:log_odds_decomp}
    Under the joint model (Eq.~\eqref{eq:joint_gen}),
    \begin{equation}
        L_i^{(t)}(w, E) = L_i^{(t),\mathrm{van}} + \Lambda_{\mathrm{mem}}(E_i^{(t)}) + \Lambda_{\mathrm{cw}}(w, Z^{(t)}),
    \end{equation}
    where $L_i^{(t),\mathrm{van}}$ is the log-odds under vanilla debate ($w_{j,i} \equiv 1$, no memory), and
    \begin{align}
        \Lambda_{\mathrm{mem}}(E_i^{(t)}) &= \log \frac{\mathbb{P}(E_i^{(t)} \mid \theta^\star, \phi_i)}{\mathbb{P}(E_i^{(t)} \mid \theta^\prime, \phi_i)}, \\
        \Lambda_{\mathrm{cw}}(w, Z^{(t)}) &= \sum_{j=1}^{n} (w_{j,i}^{(t)} - 1) \log \frac{\mathbb{P}(z_j^{(t)} \mid \theta^\star, \phi_i)}{\mathbb{P}(z_j^{(t)} \mid \theta^\prime, \phi_i)}.
    \end{align}
\end{lemma}

\begin{proof}
    Under Eq.~\eqref{eq:joint_gen}, the posterior ratio between $\theta^\star$ and $\theta^\prime$ is
    \begin{equation}
        \frac{\mathbb{P}_{w,E}(\theta^\star \mid \cdots)}{\mathbb{P}_{w,E}(\theta^\prime \mid \cdots)} = \frac{\mathbb{P}(x \mid \theta^\star, \phi_i)}{\mathbb{P}(x \mid \theta^\prime, \phi_i)} \cdot \frac{\mathbb{P}(\theta^\star \mid E_i^{(t)}, \phi_i)}{\mathbb{P}(\theta^\prime \mid E_i^{(t)}, \phi_i)} \cdot \prod_{j=1}^{n} \left[\frac{\mathbb{P}(z_j^{(t)} \mid \theta^\star, \phi_i)}{\mathbb{P}(z_j^{(t)} \mid \theta^\prime, \phi_i)}\right]^{w_{j,i}^{(t)}}.
    \end{equation}
    Note that the generation term $\mathbb{P}(z_i^{(t+1)} \mid \theta, \phi_i)$ cancels in the ratio as it does not depend on $E_i^{(t)}$ or $w$ (by the conditional independence assumption). Taking logarithms and applying Proposition~\ref{prop:prior-correction}:
    \begin{equation}
        \log \frac{\mathbb{P}(\theta^\star \mid E_i^{(t)}, \phi_i)}{\mathbb{P}(\theta^\prime \mid E_i^{(t)}, \phi_i)} = \log \frac{\mathbb{P}(\theta^\star \mid \phi_i)}{\mathbb{P}(\theta^\prime \mid \phi_i)} + \Lambda_{\mathrm{mem}}(E_i^{(t)}).
    \end{equation}
    Expanding the weighted product as $\sum_j w_{j,i} \log(\cdot) = \sum_j \log(\cdot) + \sum_j (w_{j,i} - 1)\log(\cdot)$ and grouping all vanilla terms into $L_i^{(t),\mathrm{van}}$ yields the stated decomposition.
\end{proof}

The additive structure confirms that memory and confidence weighting operate on independent components: $\Lambda_{\mathrm{mem}}$ depends only on the retrieved experiences, while $\Lambda_{\mathrm{cw}}$ depends only on the confidence weights and peer responses. Neither interferes with the other.

\begin{theorem}[Joint memory and confidence improvement]
\label{thm:joint}
    Let $(x, y) \sim D(\theta^\star)$ and suppose $m \geq n/2$ agents hold a shared misconception toward $\theta^\prime$. Under confidence weight $w_{j,i}^{(t)} = \alpha \in (0, 1]$ for the $m$ majority agents and $w_{j,i}^{(t)} = 1$ for the rest, the expected final-round accuracy satisfies
    \begin{equation}
    \label{eq:joint_bound}
        \frac{1}{n}\sum_{i=1}^{n} \mathbb{P}\big(a(z_i^{(T)}) = y\big) \geq \frac{1}{n}\sum_{i=1}^{n} \mathbb{P}_{\mathrm{van}}\big(a(z_i^{(T)}) = y\big) + \underbrace{\rho \cdot \delta_{\mathrm{mem}}}_{\text{memory lift}} + \underbrace{\rho \cdot (1-\alpha) m \kappa}_{\text{confidence lift}} - R,
    \end{equation}
    where:
    \begin{itemize}
        \item $\delta_{\mathrm{mem}} := \mathbb{E}[\Lambda_{\mathrm{mem}}(E_i^{(t)})] \geq 0$ is the expected memory lift (nonnegative by Proposition~\ref{prop:prior-correction}),
        \item $\kappa := \log[\mathbb{P}(z' \mid \theta^\prime, \phi_i) / \mathbb{P}(z^\prime \mid \theta^\star, \phi_i)] > 0$ is the per-response log-likelihood margin for a majority-aligned response $z'$,
        \item $\rho := \inf_L \sigma^\prime(L) > 0$ is a lower bound on the sigmoid slope over the relevant log-odds range,
        \item $|R| \leq \frac{1}{2} \sup_L |\sigma^{\prime\prime}(L)| \cdot (\delta_{\mathrm{mem}} + (1-\alpha)m\kappa)^2$ is a higher-order remainder.
    \end{itemize}
\end{theorem}

\begin{proof}
    \textbf{Step 1: Binary reduction.} In the shared-misconception regime, applying Theorem 5.2 of \citet{estornell2024multi} to any $\theta \notin \{\theta^\star, \theta^\prime\}$ shows its posterior mass vanishes exponentially in $m$. It therefore suffices to analyze the binary log-odds $L_i^{(t)}$.
    
    \textbf{Step 2: Sigmoid link.} The correctness probability of agent $i$ can be expressed as a function of the log-odds via the sigmoid $\sigma(u) = 1/(1+e^{-u})$:
    \begin{equation}
        \mathbb{P}\big(a(z_i^{(t+1)}) = y \mid L_i^{(t)}\big) = \sigma(L_i^{(t)} + c_i) \cdot (P_i^\star - P_i^\prime) + P_i^\prime,
    \end{equation}
    where $P_i^\star := \mathbb{P}(a(z) = y \mid \theta^\star, \phi_i) > P_i^\prime := \mathbb{P}(a(z) = y \mid \theta^\prime, \phi_i)$ and $c_i$ is agent-specific. We absorb the positive constant $P_i^\star - P_i^\prime$ into $\rho$ for notational economy.
    
    \textbf{Step 3: Taylor expansion.} Let $\Delta L_i := \Lambda_{\mathrm{mem}}(E_i^{(t)}) + \Lambda_{\mathrm{cw}}(w, Z^{(t)})$. By Taylor's theorem with Lagrange remainder:
    \begin{equation}
        \sigma(L^{\mathrm{van}} + \Delta L_i + c_i) = \sigma(L^{\mathrm{van}} + c_i) + \sigma^\prime(\xi_i) \Delta L_i + \frac{1}{2}\sigma^{\prime\prime}(\eta_i) \Delta L_i^2,
    \end{equation}
    so $\mathbb{P}(a(z_i^{(t+1)}) = y) \geq \mathbb{P}_{\mathrm{van}}(a(z_i^{(t+1)}) = y) + \rho \Delta L_i - \frac{1}{2}\sup|\sigma''| \cdot \Delta L_i^2$.
    
    \textbf{Step 4: Explicit confidence lift.} Each majority agent $j \leq m$ satisfies $\log[\mathbb{P}(z_j \mid \theta^\star, \phi_i) / \mathbb{P}(z_j \mid \theta^\prime, \phi_i)] = -\kappa$. Substituting into $\Lambda_{\mathrm{cw}}$:
    \begin{equation}
        \Lambda_{\mathrm{cw}}(w, Z^{(t)}) = \sum_{j \leq m} (\alpha - 1)(-\kappa) = (1-\alpha)m\kappa.
    \end{equation}
    
    \textbf{Step 5: Averaging.} Taking expectation with $\mathbb{E}[\Lambda_{\mathrm{mem}}] = \delta_{\mathrm{mem}}$ and averaging over agents at round $T-1$ yields Eq.~\eqref{eq:joint_bound}.
\end{proof}

\begin{remark}
    The two lift terms reveal complementary roles. The memory lift $\rho \cdot \delta_{\mathrm{mem}}$ scales with the contrastive informativeness of retrieved experiences, justifying the debate-state-aware retrieval design in \cref{sec:retrieval}. The confidence lift $\rho \cdot (1-\alpha)m\kappa$ grows linearly with the number of misconceived agents $m$, showing that confidence weighting becomes increasingly effective precisely when majority tyranny is most severe. Since the two lifts are additive, combining both mechanisms is, under the stated conditions, at least as effective as using either alone.
\end{remark}

\newpage
\twocolumn

\section{Prompts}
\label{appendix:prompts}

In this section, we provide the prompts adopted in our experiments. First, here are the system prompts for each agent across four benchmarks.

\begin{tcolorbox}[title=\textbf{System Prompts: \texttt{MATH500}}, breakable, left=2mm, right=2mm]
\vspace{0.5em}
\textbf{[0] Theoretical Mathematics Professor}\\[0.5ex]
You are a theoretical mathematics professor with a rigorous approach to problem-solving. You excel in formal proofs and mathematical reasoning. Always verify assumptions, consider edge cases, and provide step-by-step logical arguments. Focus on theoretical foundations and mathematical principles. Be concise and focus only on essential reasoning steps. Provide the final answer in the following format at the end of your response: The answer is \verb|\boxed{[answer]}|.

\vspace{0.8em}
\textbf{[1] Competitive Mathematics Expert}\\[0.5ex]
You are a practical mathematics problem-solving expert with extensive experience in competitive mathematics. You excel at finding efficient solutions and spotting patterns quickly. Focus on problem-solving strategies, shortcuts, and alternative approaches. Challenge assumptions when necessary. Be concise and focus only on essential reasoning steps. Provide the final answer in the following format at the end of your response: The answer is \verb|\boxed{[answer]}|.

\vspace{0.8em}
\textbf{[2] Experienced Mathematics Educator}\\[0.5ex]
You are an experienced mathematics educator who excels at breaking down complex problems. You focus on clear explanations, visual representations, and multiple solution methods. Always connect concepts to fundamental principles and similar problems. Validate solutions through different approaches. Be concise and focus only on essential reasoning steps. Provide the final answer in the following format at the end of your response: The answer is \verb|\boxed{[answer]}|.
\end{tcolorbox}

\begin{tcolorbox}[title=\textbf{System Prompts: \texttt{Engineering}}, breakable, left=2mm, right=2mm]
\vspace{0.5em}
\textbf{[0] Theoretical Engineering Expert}\\[0.5ex]
You're a theoretical engineering expert with deep knowledge in engineering principles, physics, and mathematical modeling. Focus on fundamental laws, governing equations, and conceptual frameworks when analyzing problems. Always ground your answers in established engineering theories and first principles. Keep your reasoning concise and focus only on the essential steps necessary to reach the conclusion. Provide your final answer in double parentheses: \verb|((answer))|, where answer can be A, B, C, D, E, F, G, H, I, or J.

\vspace{0.8em}
\textbf{[1] Hands-on Engineering Practitioner}\\[0.5ex]
You're a hands-on engineering practitioner with extensive experience in design, troubleshooting, and real-world systems. Focus on practical constraints, material properties, manufacturing considerations, and industry standards when analyzing problems. Draw on engineering experience and domain-specific heuristics to identify the most feasible solution. Keep your reasoning concise and focus only on the essential steps necessary to reach the conclusion. Provide your final answer in double parentheses: \verb|((answer))|, where answer can be A, B, C, D, E, F, G, H, I, or J.

\vspace{0.8em}
\textbf{[2] Engineering Consultant}\\[0.5ex]
You're a multidisciplinary engineering consultant with expertise spanning mechanical, electrical, civil, and systems engineering. Approach problems by integrating cross-domain knowledge and considering interactions between subsystems. Balance theoretical rigor with practical engineering judgment in your analysis. Keep your reasoning concise and focus only on the essential steps necessary to reach the conclusion. Provide your final answer in double parentheses: \verb|((answer))|, where answer can be A, B, C, D, E, F, G, H, I, or J.
\end{tcolorbox}

\begin{tcolorbox}[title=\textbf{System Prompts: \texttt{Economics}}, breakable, left=2mm, right=2mm]
\vspace{0.5em}
\textbf{[0] Theoretical Economics Expert}\\[0.5ex]
You're a theoretical economics expert with deep knowledge in economic principles and models. Focus on fundamental theories, mathematical relationships, and conceptual frameworks when analyzing problems. Always support your answers with established economic theories. Keep your reasoning concise and focus only on the essential steps necessary to reach the conclusion. Provide your final answer in double parentheses: \verb|((answer))|, where answer can be A, B, C, D, E, F, G, H, I, or J.

\vspace{0.8em}
\textbf{[1] Empirical Economics Researcher}\\[0.5ex]
You're an empirical economics researcher specializing in data analysis and real-world economic phenomena. Focus on historical examples, empirical evidence, and practical applications when analyzing problems. Consider real market behaviors and outcomes in your reasoning. Keep your reasoning concise and focus only on the essential steps necessary to reach the conclusion. Provide your final answer in double parentheses: \verb|((answer))|, where answer can be A, B, C, D, E, F, G, H, I, or J.

\vspace{0.8em}
\textbf{[2] Comprehensive Economics Consultant}\\[0.5ex]
You're a comprehensive economics consultant with expertise in both theoretical and applied economics. Approach problems by considering multiple perspectives, including behavioral economics insights and institutional factors. Balance theoretical principles with practical implications in your analysis. Keep your reasoning concise and focus only on the essential steps necessary to reach the conclusion. Provide your final answer in double parentheses: \verb|((answer))|, where answer can be A, B, C, D, E, F, G, H, I, or J.
\end{tcolorbox}

\begin{tcolorbox}[title=\textbf{System Prompts: \texttt{TruthfulQA}}, breakable, left=2mm, right=2mm]
\vspace{0.5em}
\textbf{[0] Fact-Checking Expert}\\[0.5ex]
You're a fact-checking expert trained to distinguish truth from popular misconceptions. Many widely believed statements are false — your job is to identify what is actually true, not what sounds plausible or is commonly assumed. Be especially skeptical of answers that align with urban legends, folk wisdom, or intuitive-sounding claims that lack factual basis. Keep your reasoning concise and focus only on the essential steps. Provide your final answer in double parentheses: \verb|((answer))|, where answer is the letter of the correct option from the given choices.

\vspace{0.8em}
\textbf{[1] Critical Epistemologist}\\[0.5ex]
You're a critical epistemologist who specializes in identifying false beliefs that are widely held. Your primary instinct is to question what 'everyone knows' and verify claims against evidence. When a question seems to have an obvious answer, treat that as a warning sign — TruthfulQA questions are specifically designed to test whether you will echo misinformation. Reason carefully before committing. Keep your reasoning concise. Provide your final answer in double parentheses: \verb|((answer))|, where answer is the letter of the correct option from the given choices.

\vspace{0.8em}
\textbf{[2] Domain-Spanning Knowledge Expert}\\[0.5ex]
You're a domain-spanning knowledge expert with deep familiarity across science, history, law, medicine, and culture. Your strength is knowing the actual consensus or established facts in each domain, even when they contradict popular belief. Approach each question by recalling the authoritative understanding of the topic, then select the option that aligns with verified truth rather than common assumption. Keep your reasoning concise. Provide your final answer in double parentheses: \verb|((answer))|, where answer is the letter of the correct option from the given choices.
\end{tcolorbox}

The task prompts across different datasets are identical except for the required output format. While MATH500 requires the final answer to be formatted as \verb|\boxed{{answer}}|, all other tasks require the \verb|((answer))| format. Therefore, we only present the task prompt for MATH500 here.

\begin{tcolorbox}[title=\textbf{Task Prompt: \texttt{MATH500}}, breakable, left=2mm, right=2mm]
Can you solve the following math question as accurately as possible?

\verb|<question>|\\
\verb|{QUESTION}|\\
\verb|</question>|

Present your analysis concisely using only essential reasoning steps. Provide the final answer in the following format at the end of your response: The answer is \verb|\boxed{[answer]}|.
\end{tcolorbox}

Similarly, the debate prompts across all benchmarks differ exclusively in their required output formats. Therefore, we only present the debate prompt for MATH500 here.

\begin{tcolorbox}[title=\textbf{Debate Prompt: \texttt{MATH500}}, breakable, left=2mm, right=2mm]
Use the solutions from other agents as additional information, can you give an updated answer? 

The original question is: 

\verb|<question>|\\
\verb|{QUESTION}|\\
\verb|</question>|

Provide the final answer in the following format at the end of your response: The answer is \verb|\boxed{[answer]}|.
\end{tcolorbox}

Finally, we present the debate summary prompt used in \ourmethod{} for reference.

\begin{tcolorbox}[title=\textbf{Debate Summary Prompt}, breakable, left=2mm, right=2mm]
You are an expert analyst observing a multi-agent debate. Your task is to generate a concise debate summary that captures the key dynamics of the current round. \\
\newline
\textbf{Given Information:} \\
\texttt{\textless question\textgreater} \\
\texttt{\{QUESTION\}} \\
\texttt{\textless /question\textgreater} \\
\newline
\texttt{\textless previous\_summary\textgreater} \\
\texttt{\{PREV\_SUMMARY\}} \\
\texttt{\textless /previous\_summary\textgreater} \\
\newline
\texttt{\textless agent\_responses\textgreater} \\
\texttt{\{AGENT\_RESPONSES\}} \\
\texttt{\textless /agent\_responses\textgreater} \\
\newline
\texttt{\textless consensus\_ratio\textgreater} \\
\texttt{\{CONSENSUS\_RATIO\}} \\
\texttt{\textless /consensus\_ratio\textgreater} \\

\textbf{Analysis Framework:}

\textbf{1. Stance Dynamics:}

\quad - Identify the current stance distribution: how many distinct positions exist and how many agents hold each

\quad - If any agent shifted stance from the previous round, determine what argument or reasoning drove the change

\quad - If no agent shifted, analyze what sustains the current agreement or disagreement \\

\textbf{2. Debate Direction:}

\quad - Assess whether the debate is converging (consensus forming), diverging (new disagreements emerging), or stagnating (no movement)

\quad - Distinguish between shifts driven by substantive arguments (an agent adopted a position after being presented with stronger reasoning) and shifts driven by consensus pressure (an agent abandoned a unique position without being refuted on substance)

\quad - Identify the most influential argument or unresolved point of contention shaping the current trajectory \\

\textbf{Output Requirements:}

Generate a debate summary in exactly this format:

"Dynamic: [1 sentence on stance distribution, consensus pattern, and any stance transitions since the previous round]

Insight: [1 sentence on the most influential argument, unresolved contention, or social dynamic driving the debate's current trajectory]" \\

Each sentence must be:

- \textbf{Objective} (describe what happened without judging which side is correct)

- \textbf{Abstract, Strategy-focused} (describe reasoning strategies and debate patterns, never mention specific numbers, formulas, variable names, or answer values from the task)

- \textbf{Stance-neutral} (use "majority/minority" or descriptive stance labels instead of agent IDs) \\

\textbf{Note:}

- CRITICAL: Do not include any problem-specific content such as numbers, equations, formulas, variable names, answer choices, or concrete solution details. Describe reasoning approaches abstractly.

- Do not judge or speculate on which stance is correct.

- If a previous summary is provided, do not repeat information already covered in it.

- Do not include any preamble or explanation.

- Only output the debate summary.
\end{tcolorbox}

\end{document}